\documentclass{article}
\PassOptionsToPackage{numbers, compress}{natbib}
\usepackage[utf8]{inputenc}
\usepackage{amsmath}
\usepackage{amsthm}
\usepackage{bm}
\usepackage{wrapfig}
\usepackage{graphicx}

\usepackage[colorinlistoftodos]{todonotes}

\newtheorem{lemma}{Lemma}[section]
\newtheorem{remark}{Remark}[section]

\newtheorem{theorem}{Theorem}[section]
\newtheorem{definition}{Definition}[section]

\usepackage{thm-restate}

\definecolor{citrine}{rgb}{0.89, 0.82, 0.04}
\definecolor{blued}{RGB}{70,197,221}
\definecolor{applegreen}{rgb}{0.55, 0.71, 0.0}
\definecolor{flame}{rgb}{0.89, 0.35, 0.13}
\DeclareMathOperator*{\EX}{\mathbb{E}}

\DeclareMathOperator*{\argmax}{arg\,max}

\DeclareRobustCommand{\quotes}[1]{``#1''}
\newcommand{\norm}[1]{\left\|#1\right\|}
\newcommand{\bmr}[1]{\bm{\mathrm{#1}}}

\newcommand{\Nat}[1][]{\mathbb{N}_{\ifthenelse{\isempty{#1}}{}{\ge #1}}}
\newcommand{\Reals}[1][]{\mathbb{R}_{\ifthenelse{\isempty{#1}}{}{\ge #1}}}

\newcommand{\algname}{LOGEL}
\newcommand{\mathbr}[1]{\bm{\mathbf{#1}}}
\newcommand{\vtheta}{\bm{\theta}}
\newcommand{\vomega}{\bm{\omega}}
\newcommand{\vphi}{\bm{\phi}}
\newcommand{\vvarphi}{\bm{\varphi}}

\newcommand{\traj}{\tau}
\newcommand{\vpsi}{\bm{\psi}}

\newcommand{\reals}{\mathbb{R}}

\newcommand{\identity}{\mathbr{I}}

\newcommand{\grads}{\mathbr{\Psi}}

\newcommand{\jacobian}{\nabla_{\vtheta} \vpsi}
\newcommand{\estimatedjacobian}{\widehat{\nabla_{\vtheta} \vpsi}}
\usepackage{algorithm}
\usepackage{algorithmic}
\usepackage{xcolor}
 \usepackage[preprint]{neurips_2020}
\usepackage{amssymb}
\usepackage[utf8]{inputenc} 
\usepackage[T1]{fontenc}    
\usepackage{hyperref}       
\usepackage{url}            
\usepackage{booktabs}       
\usepackage{amsfonts}       
\usepackage{nicefrac}       
\usepackage{microtype}      

\title{Inverse Reinforcement Learning from a Gradient-based Learner}

\author{%
  Giorgia Ramponi \\
   Politecnico Di Milano \\
   \texttt{giorgia.ramponi@polimi.it} \\
   \And
   Gianluca Drappo \\
  Politecnico Di Milano \\
   \texttt{gianluca.drappo@mail.polimi.it} \\
   \And
   Marcello Restelli \\
 Politecnico Di Milano \\
   \texttt{marcello.restelli@polimi.it} \\
}

\begin{document}

\maketitle

\begin{abstract}
Inverse Reinforcement Learning addresses the problem of inferring an expert's reward function from demonstrations. However, in many applications, we not only have access to the expert's near-optimal behavior, but we also observe part of her learning process.
In this paper, we propose a new algorithm for this setting, in which the goal is to recover the reward function being optimized by an agent, given a sequence of policies produced during learning. Our approach is based on the assumption that the observed agent is updating her policy parameters along the gradient direction. Then we extend our method to deal with the more realistic scenario where we only have access to a dataset of learning trajectories. For both settings, we provide theoretical insights into our algorithms' performance. Finally, we evaluate the approach in a simulated GridWorld environment and on the MuJoCo environments, comparing it with the state-of-the-art baseline.
\end{abstract}

\section{Introduction}

Inverse Reinforcement Learning (IRL)~\cite{ng2000algorithms} aims to infer an expert's reward function from her demonstrations~\cite{DBLP:journals/ftrob/OsaPNBA018}. In the standard setting, an expert shows a behavior by repeatedly interacting with the environment. This behavior, encoded by its policy, is optimizing an \emph{unknown} reward function. The goal of IRL consists of finding a reward function that makes the expert's behavior optimal~\cite{ng2000algorithms}. Compared to other imitation learning approaches~\cite{argall2009survey, hussein2017imitation}, which output an imitating policy (e.g, Behavioral Cloning~\cite{argall2009survey}), IRL explicitly provides a succinct representation of the expert's \emph{intention}. For this reason it provides a generalization of the expert's policy to unobserved situations.

However, in some cases, it is not possible to wait for the convergence of the learning process. For instance, in multi-agent environments, an agent has to infer the unknown reward functions that the other agents are learning, before actually becoming \quotes{experts}; so that she can either cooperate or compete with them. On the other hand, in many situations, we can learn something useful by observing the learning process of an agent. These observations contain important information about the agent's intentions and can be used to infer her interests. Imagine a driver who is learning a new circuit. During her training, we can observe how she behaves in a variety of situations (even dangerous ones) and this is useful for understanding which states are good and which should be avoided. Instead, when expert behavior is observed, only a small sub-region of the state space could be explored, thus leaving the observer unaware of what to do in situations that are unlikely under expert policy. 


Inverse Reinforcement Learning from not expert agents, called Learning from a Learner (LfL), was recently proposed by Jacq et Al. in~\cite{lflpaper}. LfL involves two agents: a \textit{learner} who is currently learning a task and an \textit{observer} who wants to infer the learner's intentions. In~\cite{lflpaper} the authors assume that the learner is learning under an entropy-regularized framework, motivated by the assumption that the learner is showing a sequence of constantly improving policies. However many Reinforcement Learning (RL) algorithms~\cite{sutton1998reinforcement} do not satisfy this and also human learning is characterized by mistakes that may lead to a non-monotonic learning process. 

In this paper we propose a new algorithm for the LfL setting called Learning Observing a Gradient not-Expert Learner (\algname), that is less affected by the violation of the constantly improving assumption.
Given that many successful RL algorithms are gradient-based~\cite{peters2006policy} and there is some evidence that the human learning process is similar to a gradient-based method~\cite{shteingart2014reinforcement}, we assume that the learner is following the gradient direction of her expected discounted return.
The algorithm learns the reward function that minimizes the distance between the actual policy parameters of the learner and the policy parameters that should be if she were following the policy gradient using that reward function. 

After a formal introduction of the LfL setting in Section~\ref{S:IRL-LFL}, we provide in Section~\ref{sec:exact} a first solution of the LfL problem when the observer has full access to the learner's policy parameters and learning rates. Then, in Section~\ref{sec:learndem} we extend the algorithm to the more realistic case in which the observer can identify the optimized reward function only by analyzing the learner's trajectories. For each problem setting, we provide a finite sample analysis in order to give to the reader an intuition on the correctness of the recovered weights. Finally, we consider discrete and continuous simulated domains in order to empirically compare the proposed algorithm with the state-of-the-art baseline in this setting~\cite{lflpaper}. The proofs of all the results are reported in Appendix~\ref{apx:proofs}.


\section{Preliminaries}
\label{sec:2}
A \textbf{Markov Decision Process} (MDP)~\cite{puterman1994markov, sutton1998reinforcement} is a tuple $\mathcal{M} = \left( \mathcal{S}, \mathcal{A}, {P}, \gamma, \mu, R \right)$
where $\mathcal{S}$ is the state space, $\mathcal{A}$ is the action space, ${P}: \mathcal{S \times A \times S} \rightarrow \mathbb{R}_{\ge 0}$  is the transition function, which defines the density ${P}(s' | s,a)$ of state $s' \in \mathcal{S}$ when taking action $a \in \mathcal{A}$ in state $s \in \mathcal{S}$, $\gamma \in [0,1)$ is the discount factor, $\mu: \mathcal{S} \rightarrow \mathbb{R}_{\ge 0}$ is the initial state distribution and $R: \mathcal{S} \rightarrow \reals$ is the reward function. An RL agent follows a policy $\pi : \mathcal{S \times A} \rightarrow \mathbb{R}_{\ge 0}$, where $\pi(\cdot | s)$ specifies for each state $s$ a distribution over the action space $\mathcal{A}$, i.e., the probability of taking action $a$ in state $s$.
We consider stochastic differentiable policies belonging to a parametric space $\Pi_{\Theta} = \{ \pi_{\vtheta} : \vtheta \in \Theta \subseteq \mathbb{R}^d \}$. 
We evaluate the performance of a policy $\pi_{\vtheta}$ as its expected cumulative discounted return:
\begin{equation*}
    J(\vtheta)=\EX_{\substack{S_0 \sim \mu,\\ A_t \sim \pi_{\vtheta}(\cdot | S_t), \\ S_{t+1} \sim {P}(\cdot | S_t, A_t) }}[\sum_{t=0}^\infty \gamma^t R(S_t,A_t)].
\end{equation*}
To solve an MDP, we must find a policy $\pi_{\vtheta^*}$ that maximizes the performance $\vtheta^* \in \argmax_{\vtheta} J(\vtheta)$.

\textbf{Inverse Reinforcement Learning}~\cite{DBLP:journals/ftrob/OsaPNBA018, ng2000algorithms, abbeel2004apprenticeship} addresses the problem of recovering the \emph{unknown} reward function optimized by an expert given demonstrations of her behavior. The expert plays a policy $\pi^E$ which is (nearly) optimal for some unknown reward function ${R} : \mathcal{S} \times \mathcal{A} \rightarrow \mathbb{R}$. We are given a dataset ${D} = \{\traj_1,\dots,\traj_n\}$ of trajectories from $\pi^E$, where we define a \textit{trajectory} as a sequence of states and actions $\traj = (s_0, a_0,\dots,s_{T-1}, a_{T-1}, s_T)$, where $T$ is the trajectory length. 
The goal of an IRL agent is to find a reward function that explains the expert's behavior. As commonly done in the Inverse Reinforcement Learning literature \cite{pirotta2016inverse,ziebart2008maximum, abbeel2004apprenticeship}, we assume that the expert's reward function can be represented by a linear combination with weights $\vomega$ of $q$ basis functions $\vphi$:
\begin{equation}\label{eq:LinRew}
	{R}_{\vomega}(s,a) = \vomega^T \vphi(s,a), \;\; \vomega \in \reals^q,
\end{equation} 
with $\vphi: \mathcal{S} \times \mathcal{A} \rightarrow [-M_r,M_r]$ is a limited feature vector function. 

We define the \textbf{feature expectations} of a policy $\pi_{\vtheta}$ as:
\begin{equation*}
\bmr{\psi}(\vtheta) = 
    \EX_{\substack{S_0 \sim \mu,\\ A_t \sim \pi_{\vtheta}(\cdot | S_t), \\ S_{t+1} \sim {P}(\cdot | S_t, A_t) }}\left[ \sum_{t=0}^{+\infty} \gamma^t \vphi(S_t,A_t)\right].
\end{equation*}

The \textbf{expected discounted return}, under the linear reward model, is defined as:
\begin{equation}
\displaystyle     J(\vtheta, \vomega)  = \EX_{\substack{S_0 \sim \mu,\\ A_t \sim \pi_{\vtheta} (\cdot | S_t), \\ S_{t+1} \sim {P}(\cdot | S_t, A_t) }}\left[ \sum_{t=0}^{+\infty} \gamma^t {R_{\vomega}}(S_t,A_t) \right] = \vomega^T \vpsi(\vtheta).
\end{equation}


\section{Inverse Reinforcement Learning from learning agents}\label{S:IRL-LFL}
The Learning from a Learner Inverse Reinforcement Learning setting (LfL), proposed in~\cite{lflpaper}, involves two agents:
\begin{itemize}
 \item a \textit{learner} which is learning a task defined by the reward function $R_{\vomega^L}$, 
 \item  and an \textit{observer} which wants to infer the learner's reward function.
\end{itemize} 
More formally, the learner is an RL agent which is learning a policy $\pi_{\vtheta} \in \Pi_{\Theta}$ in order to maximize its \emph{discounted expected return} $J(\vtheta, \vomega^L)$. The learner is improving its own policy by an update function $f(\vtheta, \vomega): \reals^{d \times q} \rightarrow \reals^d$, i.e., at time $t$, $\vtheta_{t+1} = f(\vtheta_{t}, \vomega)$. The observer, instead, perceives a sequence of learner's policy parameters $\{\vtheta_1, \cdots, \vtheta_{m+1}\}$ and a dataset of trajectories for each policy $D = \{D_1, \cdots, D_{m+1}\}$, where $D_i = \{ \tau_1^i, \cdots, \tau_n^i\}$. Her goal is to recover the reward function $R_{\vomega^L}$ that explains $\pi_{\vtheta_{i}} \rightarrow \pi_{\vtheta_{i+1}}$ for all $1 \le i \le m$, i.e the updates of the learner's policy. 

\begin{remark}
It is easy to notice that this problem has the same intention as Inverse Reinforcement Learning, since the demonstrating agent is motivated by some reward function. On the other hand, in classical IRL the learner agent is an expert, and not as in LfL a non-stationary agent. 
For this reason, we cannot simply apply standard IRL algorithms to this problem.
\end{remark}

\section{Learning from a learner following the gradient}
\label{sec:exact}
Many policy-gradient algorithms~\cite{peters2006policy, sutton2000policy} were proposed to solve reinforcement learning tasks. This algorithm relies in gradient-based updates of the policy parameters. Recently was also proved that standard algorithms as Value Iteration and Q-Learning have strict connections with policy gradient methods \cite{goyal2019first,schulman2017equivalence}. 
For the above reasons, we assume that the learner is optimizing the expected discounted return using gradient descent. 

For the sake of presentation, we start by considering the simplified case in which we assume that the observer can perceive the sequence of the learner's policy parameters $(\vtheta_1, \cdots, \vtheta_{m+1})$, the associated gradients of the feature expectations $(\nabla_{\vtheta}\vpsi(\vtheta_{1}),\dots,\nabla_{\vtheta}\vpsi(\vtheta_{m}))$, and the learning rates $(\alpha_1, \cdots, \alpha_{m})$. Then, we will replace the exact knowledge of the gradients with estimates built on a set of demonstrations $D_i$ for each learner's policy $\pi_{\vtheta_i}$ (Section~\ref{ss:approximate}). Finally, we introduce our algorithm \algname, which, using behavioral cloning and an alternate block-coordinate optimization, is able to estimate the reward's parameters without requiring as input the policy parameters and the learning rates (Section~\ref{sec:learndem}).

\subsection{Exact gradient}\label{ss:exact}

We express the gradient of the expected return as \cite{sutton2000policy, peters2008reinforcement}:
\begin{equation*}
	\nabla_{\vtheta} J(\mathbr{\theta},\mathbr{\omega}) =  \EX_{\substack{S_0 \sim \mu,\\ A_t \sim \pi_{\vtheta}(\cdot | S_t), \\ S_{t+1} \sim {P}(\cdot | S_t, A_t) }}  \bigg[ \sum_{t=0}^{+\infty} \gamma^t {R}_{\vomega}(S_t,A_t) \sum_{l=0}^t \nabla_{\vtheta}\log \pi_{\vtheta}(A_l|S_l) \bigg]  = \nabla_{\vtheta} \vpsi(\vtheta) \vomega,
\end{equation*}
where $\nabla_{\vtheta} \vpsi(\vtheta) = \left(\nabla_{\vtheta} \psi_1(\vtheta) |\dots|\nabla_{\vtheta} \psi_q(\vtheta)  \right) \in \reals^{d \times q}$ is the Jacobian matrix of the feature expectations $\vpsi(\vtheta)$ w.r.t the policy parameters $\vtheta$. In the rest of the paper, with some abuse of notation, we will indicate $\vpsi(\vtheta_t)$ with $\vpsi_t$.

We define the gradient-based learner updating rule at time $t$ as:
\begin{equation}
\label{eq:updating}
    \vtheta_{t+1}^L = \vtheta_{t}^L + \alpha_t \nabla_{\vtheta} J(\vtheta_{t}^L, \vomega) = \vtheta_{t}^L + \alpha_t \nabla_{\vtheta} \vpsi_{t}^L \vomega^L,
\end{equation}
where $\alpha_t$ is the learning rate. Given a sequence of consecutive policy parameters $(\vtheta_1^L, \cdots, \vtheta_{m+1}^L)$, and of learning rates $(\alpha_1, \cdots, \alpha_m)$ the observer has to find the reward function $R_{\vomega}$ such that the improvements are explainable by the update rule in Eq.~\eqref{eq:updating}. This implies that the observer has to solve the following minimization problem:
\begin{equation}
\label{min:sumomegas}
    \min_{\vomega \in \reals^d} \sum_{t=1}^{m} \norm{\Delta_t - \alpha_t \nabla_{\vtheta} \vpsi_{t}\vomega}_2^2,
\end{equation}
where $\Delta_t = \vtheta_{t+1} - \vtheta_{t}$. This optimization problem can be easily solved in closed form under the assumption that $ \left( \sum_{t=1}^{m} \jacobian_t \right)^T \left( \sum_{t=1}^{m} \jacobian_t \right)$ is invertible. 
\begin{restatable}[]{lemma}{weightclosedformsum}
\label{lemma:weightclosedformsum}
If the matrix $\sum_{t=1}^{m} \alpha_t\nabla_{\vtheta} \vpsi_t$ is full-rank than optimization problem ~\eqref{min:sumomegas} is solved in closed form by 
\begin{equation}\label{eq:omega}
    \widehat\vomega =\left( \sum_{t=1}^{m} \alpha_t^2\nabla_{\vtheta} \vpsi_t^T \nabla_{\vtheta} \vpsi_t\right)^{-1} \left(\sum_{t=1}^{m} \alpha_t \nabla_{\vtheta} \vpsi_t^T \Delta_t\right).
\end{equation}
\end{restatable} 
When problem~\eqref{min:sumomegas} has no unique solution or when the matrix to be inverted is nearly singular, in order to avoid numerical issues, we can resort to a regularized version of the optimization problem. In the case we add an L2-norm penalty term over weights $\vomega$ we can still compute a closed-form solution (see Lemma~\ref{lemma:regularized} in Appendix \ref{apx:proofs}).

\subsection{Approximate gradient}\label{ss:approximate}
In practice, we do not have access to the Jacobian matrix $\nabla_{\vtheta} \vpsi$, but the observer has to estimate it using the dataset $D$ and some unbiased policy gradient estimator, such as REINFORCE~\cite{williams1992simple} or G(PO)MDP~\cite{baxter2001infinite}. The estimation of the Jacobian will introduce errors on the optimization problem~\eqref{min:sumomegas}. Obviously more data are available to estimate the gradient more accurate the estimation of the reward weights $\vomega$ could be \cite{pirotta2013adaptive}. On the other hand during the learning process, the learner will produce more than one policy improvement, and the observer can use them to get better estimates of the reward weights.

In order to have an insight on the relationship between the amount of data needed to estimate the gradient and the number of learning steps, we provide a finite sample analysis on the norm of the difference between the learner's weights $\vomega^L$ and the recovered weights $\widehat\vomega$. The analysis takes into account the learning steps data and the gradient estimation data, without having any assumption on the policy of the learner. We denote with $\grads = \left[ \jacobian_1, \cdots, \jacobian_m\right]^T$ the concatenation of the Jacobians and $\widehat{\grads} = \left[ \estimatedjacobian_1, \cdots, \estimatedjacobian_m\right]^T$ the concatenation of the estimated Jacobians. 
\begin{restatable}[]{thm}{finitesampletwo}
\label{finitesampletwo}
Let $\grads$ be the real Jacobians and $\widehat{\grads}$ the estimated Jacocobian from $n$ trajectories $\{ \tau_1, \cdots, \tau_n\}$. Assume that $\grads$ is bounded by a constant $M$ and $\lambda_{\min} (\widehat{\grads}^T \widehat{\grads}) \ge \lambda > 0 $. Then w.h.p.: 
\begin{equation*} 
    \norm{\vomega^L - \widehat{\vomega}}_2 \le O \left( \frac{1}{\lambda} M \sqrt{\frac{dq\log(\frac{2}{\delta})}{2n}}  \left ( \sqrt{\frac{\log dq}{m}} + \sqrt{dq}\right ) \right).
\end{equation*}
\end{restatable}
We have to underline that a finite sample analysis is quite important for this problem. In fact, the number of policy improvement steps of the learner is finite as the learner will eventually achieve an optimal policy. So, knowing the finite number of learning improvements $m$, we can estimate how much data we need for each policy to get an estimate with a certain accuracy. More information about the proof of the theorem can be found in appendix \ref{apx:proofs}.

\begin{remark}
Another important aspect to take into account is that there is an intrinsic bias~\cite{mcwilliams2014fast} due to the gradient estimation error that cannot be solved by increasing the number of learning steps, but only with a more accurate estimation of the gradient. However, we show in Section~\ref{sec:experiments} that, experimentally, the component of the bound that does not depend on the number of learning steps does not influence the recovered weights.
\end{remark}

\section{Learning from improvement trajectories}
\label{sec:learndem}
In a realistic scenario, the observer has access only to a dataset $D = (D_1, \dots, D_{m+1})$ of trajectories generated by each policy, such that $D_i = \{\tau_1, \cdots, \tau_n\} \sim \pi_{\vtheta_i}$. Furthermore, the learning rates are unknown and possibly the learner applies an update rule other than~\eqref{eq:updating}. The observer has to infer the policy parameters $\Theta = (\vtheta_1, \dots, \vtheta_{m+1})$, the learning rates $A = (\alpha_1, \dots, \alpha_{m})$, and the reward weights $\vomega$. If we suppose that the learner is updating its policy parameters with gradient ascent on the discounted expected return, the natural way to see this problem is to maximize the log-likelihood of  $p(\vtheta_1, \vomega, A|D)$:
\begin{equation*}
\label{eq:max_lik}
    \max_{\vtheta_1, \vomega, A}  \sum_{(s,a) \in D_1} \log \pi_{\vtheta_1}(a|s) + \sum_{t=2}^{m+1} \sum_{(s,a) \in D_t} \log \pi_{\vtheta_{t}} (a|s),
\end{equation*}
where $\vtheta_t = \vtheta_{t-1} + \alpha_{t-1} \nabla_{\vtheta} \vpsi_{t-1}$. Unfortunately, solving this problem directly is not practical as it involves evaluating gradients of the discounted expected return up to the $m$-th order. To deal with this, we break down the inference problem into two steps: the first one consists in recovering the policy parameters $\Theta$ of the learner and the second in estimating the learning rates $A$ and the reward weights $\vomega$ (see Algorithm~\ref{alg:irl-mi}).

\subsection{Recovering learner policies}
Since we assume that the learner's policy belongs to a parametric policy space $\Pi_\Theta$ made of differentiable policies, as explained in~\cite{pirotta2016inverse}, we can recover an approximation of the learner's parameters $\Theta$ through behavioral cloning, exploiting the trajectories in $D = \{D_1, \cdots, D_{m+1}\}$. For each dataset $D_i \in D$ of trajectories, we cast the problem of finding the parameter $\vtheta_i$ to a maximum-likelihood estimation. Solving the following optimization problem we obtain an estimate $\widehat{\vtheta_i}$ of $\vtheta_i$:
\begin{equation}
\label{eq:bc}
	 \max_{\vtheta_i \in \Theta} \frac{1}{n} \sum_{l=1}^n \sum_{t=0}^{T-1} \log \pi_{\vtheta_i}(a_{l,t}|s_{l,t}).
\end{equation}
It is known that the maximum-likelihood estimator is consistent under mild regularity conditions on the policy space $\Pi_{\Theta}$ and assuming the identifiability property~\cite{casella2002statistical}. Some finite-sample guarantees on the concentration of distance $\|\widehat{\vtheta}_i - {\vtheta}_i \|_p$ were also derived under stronger assumptions, e.g., in~\cite{spokoiny2012parametric}.

\subsection{Recovering learning rates and reward weights}
Given the parameters $(\widehat\vtheta_1,\dots,\widehat\vtheta_{m+1})$, if the learner is updating her policy with a constant learning rate we can simply apply~Eq.~\eqref{min:sumomegas}. On the other hand, with an unknown learner, we cannot make this assumption and it is necessary to estimate also the learning rates $A=(\alpha_1,\dots,\alpha_m)$. The optimization problem in Eq.~\eqref{min:sumomegas} becomes:
\begin{align}
\label{min:alphaomega}
      \min_{\vomega \in \reals^q, A \in \reals^{m}} \sum_{t=1}^{m} \norm{\widehat\Delta_t - \alpha_t \estimatedjacobian_t \vomega}_2^2 \\
      \text{s.t.} \quad \alpha_t \geq \epsilon \quad 1 \leq t \leq m.
\end{align}
where $\widehat\Delta_t = \widehat\vtheta_{t+1}-\widehat\vtheta_t$ and $\epsilon$ is a small constant.
To optimize this function we use an alternate block-coordinate descent. We alternate the optimization of parameters $A$ and the optimization of parameters $\vomega$. Furthermore, we can notice that these two steps can be solved in closed form. When we optimize on $\vomega$, optimization can be done using Lemma~\ref{weightclosedformsum}. When we optimize on $A$ we can solve for each parameter $\alpha_t$, with $1 \le t \le m$, in closed form.
 \begin{lemma}
 The minimum of~\eqref{min:alphaomega} respect to $\alpha_t$ is equal to:
 \begin{equation}\label{eq:alpha}
     \hat\alpha_t = \max\left(\epsilon, \left((\estimatedjacobian_t \vomega)^T (\estimatedjacobian_t \vomega)\right)^{-1}(\estimatedjacobian_t \vomega )^T \hat\Delta_t\right).
 \end{equation}
 \end{lemma}
The problem cannot be inverted only if the vector $\estimatedjacobian \vomega$ is equal to $\mathbr{0}$. This would happen only if the expert is optimal, so the $\estimatedjacobian$ is $\mathbr{0}$.
The optimization converges under the assumption that there exists a unique minimum for each variable $A$ and $\vomega$~\cite{tseng2001convergence}.

\subsection{Theoretical analysis}
In this section, we provide a finite-sample analysis of \algname~when only one learning step is observed, assuming that the Jacobian matrix $\estimatedjacobian$ is bounded and the learner's policy is a Gaussian policy $\pi \sim \mathcal{N}(\theta, \sigma)$. The analysis evaluates the norm of the difference between the learner's weights $\vomega^L$ and the recovered weights $\hat{\vomega}$. Without loss of generality, we consider the case where the learning rate $\alpha = 1$. The analysis takes into account the bias introduced by the behavioral cloning and the gradient estimation.

\begin{restatable}[]{thm}{finitesamplegeneral}
\label{}
Given a dataset $D = \{\tau_1, \cdots, \tau_n\}$ of trajectories such that every trajectory $\tau_t = \{(s_1,a_1), \cdots, (s_T,a_T)\}$ is sampled from a Gaussian linear policy $\pi_{\vtheta}(\cdot|s) \sim \mathcal{N}(\vtheta^T\vvarphi(s), \sigma)$, such that $S \in \reals^{n \times d}$ is the matrix of states features, let the minimum singular value of $\sigma_{\min}(S^TS) \ge \eta > 0$, the $\estimatedjacobian$ uniformly bounded by $M$, the state features bounded by $M_{S}$, and the reward features bounded by $M_R$. Then with probability $1 - \delta$:
\begin{equation*}
     \norm{\vomega^L - \widehat{\vomega}}_2 \le O\left(\frac{(M+M^2_{S}M_R)}{\sigma_{\min}(\jacobian)} \sqrt{\frac{\log(\frac{2}{\delta})}{n\eta}}\right)
\end{equation*}
where $\omega^L$ are the real reward parameters and $\widehat{\omega}$ are the parameters recovered using Lemma~\ref{lemma:weightclosedformsum}.
\end{restatable}

\begin{algorithm}[t]
\caption{\algname} \label{alg:irl-mi}
\begin{algorithmic}[1]
\REQUIRE Dataset $\mathcal{D} = \{ \mathcal{D}_1, \dots, \mathcal{D}_{m+1}\}$ with $\mathcal{D}_j = \{ (\traj_1,\dots,\traj_{n_j}) \ |\ \traj_i \sim \pi_{\vtheta_j}\}$ 
\ENSURE Reward weights ${\vomega} \in \reals^d$
\vspace{0.1cm}
\STATE{Estimate  policy parameters $(\hat\vtheta_1,\dots,\hat\vtheta_{m+1})$ with Eq.~\eqref{eq:bc}}
\STATE{Initialize $A$ and $\vomega$}
\STATE{Compute learning rates $A$ and reward weights $\vomega$ by alternating~\eqref{eq:alpha} and~\eqref{eq:omega} up to convergence}

\end{algorithmic}
\end{algorithm}
The theorem, that relies on perturbation analysis and on least squares with fixed design, underlines how \algname, with a sufficient number of samples to estimate the policy parameters and the gradients, succeeds in recovering the correct reward parameters.

\section{Related Works}
The problem of estimating the reward function of an agent who is learning is quite new. This setting was proposed by Jacq et Al.~\cite{lflpaper} and, to the best of our knowledge, it is studied only in that work. In~\cite{lflpaper} the authors proposed a method based on entropy-regularized reinforcement learning, in which they assumed that the learner is performing soft policy improvements. In order to derive their algorithm, the authors also assume that the learner respects the policy improvement condition. We do not make this assumption as our formulation assumes only that the learner is changing its policy parameters along the gradient direction (which can result in a performance loss). 

The problem, as we underlined in Section~\ref{S:IRL-LFL}, is close to the Inverse Reinforcement Learning problem~\cite{DBLP:journals/ftrob/OsaPNBA018, ng2000algorithms}, since they share the intention of acquiring the unknown reward function from the observation of an agent's demonstrations. \algname~relies on the assumption that the learner is improving her policy through gradient ascent updates. A similar approach, but in the expert case, was taken in~\cite{pirotta2013adaptive, metelli2017compatible} where the authors use the null gradient assumption to learn the reward from expert's demonstrations. 

In another line of works, policy sub-optimal demonstrations are used in the preference-based IRL~\cite{christiano2017deep, ibarz2018reward} and ranking-based IRL \cite{brown2019extrapolating, castro2019inverse}. Some of these works require that the algorithm asks a human to compare possible agent's trajectories in order to learn the underlying reward function of the task. We can imagine that \algname~can be used in a similar way to learn from humans who are learning a new task. 
Instead, in~\cite{balakrishna2019policy} was proposed an Imitation Learning setting where the observer tries to imitate the behavior of a supervisor that demonstrates a converging sequence of policies.

In works on the \textit{theory of minds}~\cite{rabinowitz2018machine, shum2019theory}, the authors propose an algorithm that uses meta-learning to build a system that learns how to model other agents. In these works it is not required that agents are experts but they must be stationary. Instead, in the setting considered by \algname, the observed agent is non-stationary. 

\section{Experiments}
\begin{figure}
    \centering
    \includegraphics[scale=.87]{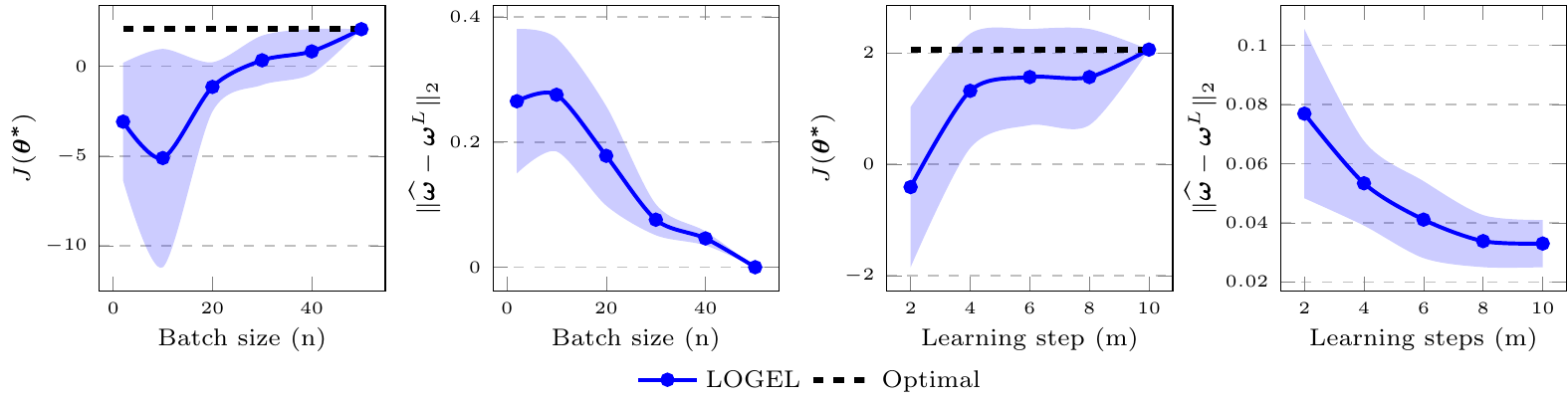}
    \caption{Gridworld experiment with known policy parameters. The learner is using G(PO)MDP algorithm. From left the expected discounted return and the norm difference between the real weights and the recovered ones with one learning step; the same measures with fixed batch size ($5$ trajectories with length $20$). The performance of the observers are evaluated on the learner’s reward weights. Results are averaged over 20 runs. $98\%$ c.i as shaded area.}
    \label{fig:exact}
\end{figure}

\label{sec:experiments}
This section is devoted to the experimental evaluation of \algname. The algorithm \algname~is compared to the state-of-the-art baseline Learner From a Learner (LfL)~ \cite{lflpaper} in a gridworld navigation task and in two MuJoCo environments. More details on the experiments are in Appendix \ref{app:experiments}.

\subsection{Gridworld}
\begin{wrapfigure}{R}{0.35\textwidth}
\centering
\includegraphics[scale=0.65]{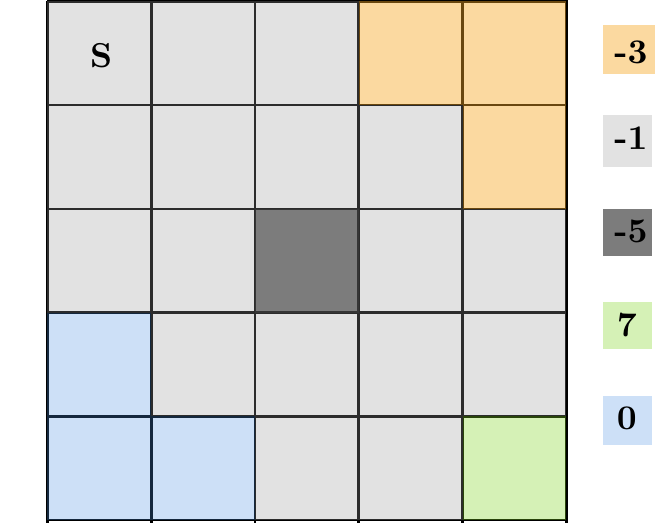}
\caption{Gridworld environment: every area has a different reward weight. In the green area the agent is reset to the starting state.}
\label{fig:grid}
\end{wrapfigure}
The first set of experiments aims at evaluating the performance of \algname~in a discrete Gridworld environment. The gridworld, represented in Figure~\ref{fig:grid}, is composed of five regions with a different reward for each area. The agent starts from the cell $(1,1)$ and when she reaches the green state, then returns to the starting state. The reward feature space is composed of the one-hot encoding of five areas: the orange, the light grey, the dark grey, the blue, and the green. The learner weights for the areas are $(-3,-1,-5,7,0)$ respectively. 
\begin{figure}
    \centering
    \includegraphics[scale=.9]{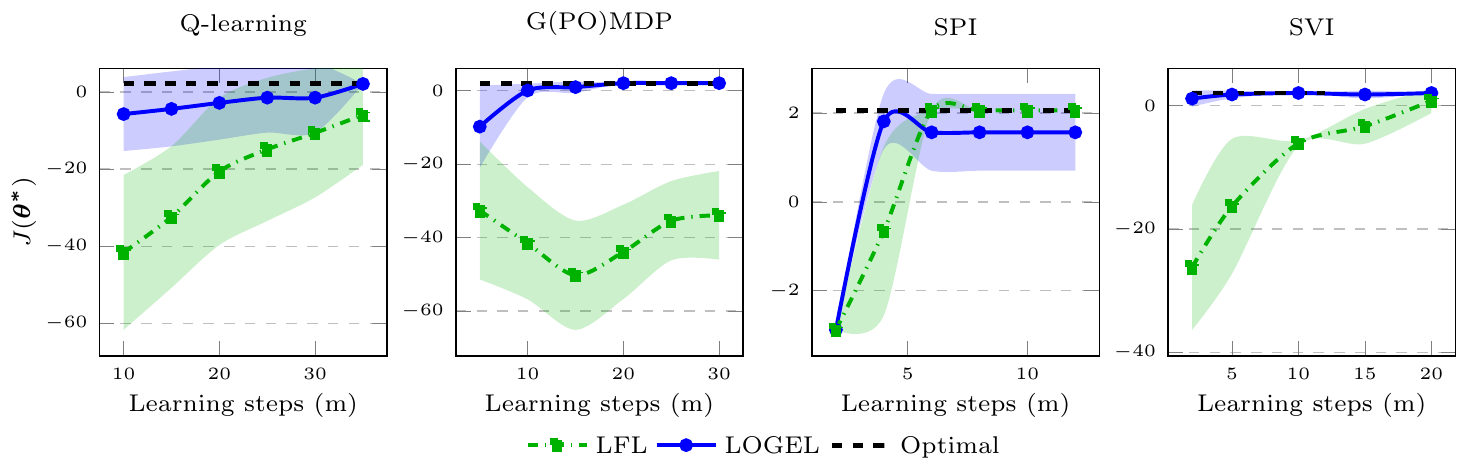}
    \caption{Gridworld experiment with estimated policy parameters and four learner: from left QLearning, G(PO)MDP, SPI, SVI. The green line is the LfL observer and the blue one is the \algname~observer. The performance of the observers are evaluated on the learner’s reward weights. Results are everaged over 20 runs. $98\%$ c.i as shaded area.} 
    \label{fig:comparison}
\end{figure}
As a first experiment, we want to verify in practice the theoretical finding exposed in Section~\ref{sec:exact}. In this experiment, the learner uses a Boltzmann policy and she is learning with the G(PO)MDP policy gradient algorithm. The observer has access to the true policy parameters of the learner. Figure~\ref{fig:exact} shows the performance of \algname~in two settings: a single learning step and increasing batch size ($5$,$10$,$20$,$30$,$40$,$50$); a fixed batch size (batch size $5$ and trajectory length $20$) and an increasing number of learning steps ($2$, $4$, $6$, $8$, $10$). The figure shows the expected discounted return (evaluated in closed form) and the difference in norm between the learner's weights and the recovered weights \footnote{To perform this comparison, we normalize the recovered weights and the learner's weights}. We note that, as explained in Theorem~\ref{finitesampletwo}, with a more accurate gradient estimate, the observer succeeds in recovering the reward weights by observing even just one learning step. On the other hand, as we can deduce from Theorem~\ref{finitesampletwo}, if we have a noisy estimation of the gradient, with multiple learning steps, the observer succeeds in recovering the learner's weights. It is interesting to notice that, from this experiment, it seems that the bias component, which does not vanish as the learning steps increase (see Theorem~\ref{finitesampletwo}), does not affect the correctness of the recovered weights.

In the second experiment we consider four different learners using: QLearning~\cite{sutton1998reinforcement}, G(PO)MDP, Soft policy improvement (SPI)~\cite{lflpaper} and Soft Value Iteration (SVI)~\cite{haarnoja2018soft} \footnote{In Appendix~\ref{app:experiments} the learning process of each learning agent is shown.}. For this experiment we compare the performance of \algname~and LfL~\cite{lflpaper}. In Figure~\ref{fig:comparison} we can notice as \algname~succeeds in recovering the learner's reward weights even with learner algorithms other than gradient-based ones.
Instead, LfL does not recover the reward weights of the G(PO)MDP learner and needs more learning steps than \algname~to learn the reward weights when Q-learning learner and SVI are observed.

\begin{wrapfigure}{R}{0.5\textwidth}
    \includegraphics[scale=0.88]{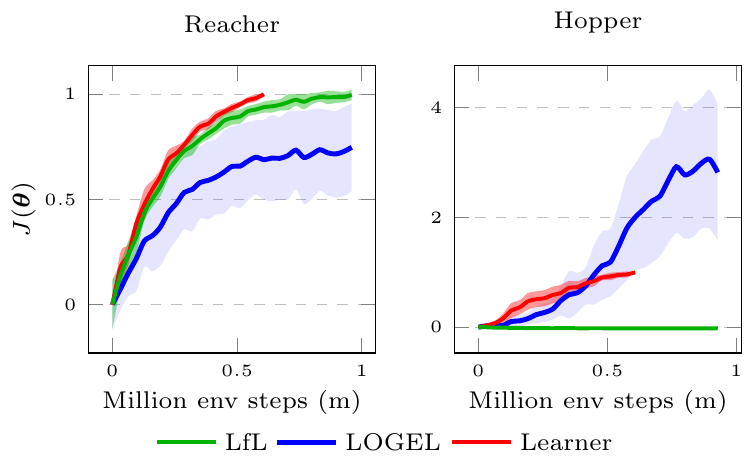}
    \caption{From the left, the Reacher and the Hopper MuJoCo environments. The red line is the performance of the learner during $20$ learning steps. The observers, LfL and \algname, observe the trajectories of the last $10$ learning steps. The performance of the observers are evaluated on the learner's reward weights. Scores are normalized setting to $0$ the first return of the learner and to $1$ the last one. The results are averaged over $10$ runs. $98\%$ c.i. are shown as shaded areas.}
    \label{fig:mujoco}
\end{wrapfigure}

\subsection{MuJoCo environments}

In the second set of experiments, we show the ability of \algname~to infer the reward weights in more complex and continuous environments. We use two environments from the MuJoCo control suite~\cite{DBLP:journals/corr/BrockmanCPSSTZ16}: Hopper and Reacher. As in~\cite{lflpaper}, the learner is trained using Policy Proximal Optimization (PPO)~\cite{schulman2017proximal}, with $16$ parallel agents for each learning step. For each step, the length of the trajectories is $2000$. Then we use \algname~or LfL to recover the reward parameters. In the end, the observer is trained with the recovered weights using PPO and the performances are evaluated on the learner's weights, starting from the same initial policy of the learner for a fair comparison. The scores are normalized by setting to $1$ the score of the last observed learner policy and to $0$ the score of the initial one. In both environments, the observer learns using the learning steps from $10$ to $20$ as the first learning steps are too noisy.
The reward function of LfL is the same as the one used in the original paper, where the reward function is a neural network equal to the one used for the learner's policy. Instead, for \algname~we used linear reward functions derived from state and action features. The reward for the Reacher environment is a $26$-grid radial basis function that describes the distance between the agent and the goal, plus the 2-norm squared of the action. In the Hopper environment, instead, the reward features are the distance between the previous and the current position and the 2-norm squared of the action.

The results are shown in Figure~\ref{fig:mujoco}, where we reported results averaged over $10$ runs. We can notice that \algname~succeeds in identifying a good reward function in both environments, although in the Reacher environment the recovered reward function causes slower learning. Instead, LfL fails to recover an effective reward function for the Hopper environment \cite{lflpaper}.

\section{Conclusions}
In this paper we propose a novel algorithm, \algname, for the \quotes{Learning from a Learner Inverse Reinforcement Learning} setting. The proposed method relies on the assumption that the learner updates her policy along the direction of the gradient of the expected discounted return. We provide some finite-sample bounds on the algorithm performance in recovering the reward weights when the observer observes the learner's policy parameters and when the observer observes only the learner's trajectories. Finally, we tested \algname~on a discrete gridworld environment and on two MuJoCo continuous environments, comparing the algorithm with the state-of-the-art baseline~\cite{lflpaper}. As future work, we plan to extend the work to account for the uncertainty in estimating both the policy parameters and the gradient.

\section*{Broader impact}
In this paper, we focus on the Inverse Reinforcement Learning \cite{abbeel2004apprenticeship, hussein2017imitation, argall2009survey, DBLP:journals/ftrob/OsaPNBA018} task from a Learning Agent \cite{lflpaper}.
The first motivation to study Inverse Reinforcement Learning algorithms is to overcome the difficulties that can arise in specifying the reward function from human and animal behavior. Sometimes, in fact, it is easier to infer human intentions by observing their behaviors than to design a reward function by hand. An example is helicopter flight control \cite{abbeel2007application}, in which we can observe a helicopter operator and through IRL a reward function is inferred to teach a physical remote-controlled helicopter. Another example is to predict the behavior of a real agent as route prediction tasks of taxis \cite{ziebart2008maximum, ziebart2008navigate} or anticipation of pedestrian interactions \cite{chung2010mobile} or energy-efficient driving \cite{vogel2012improving}. However, in many cases the agents are not really experts and on the other hand only expert demonstrations can not show their intention to avoid dangerous situations. We want to point out that learning what the agent wants to avoid because harmful is as important as learning his intentions. 

The possible outcomes of this research are the same as those of Inverse Reinforcement Learning mentioned above, avoiding the constraint that the agent has to be an expert. In future work, we will study how to apply the proposed algorithm in order to infer the pilot's intentions when they learn a new circuit.

A relevant  possible complication of using IRL is the error on the reward feature engineering which can lead to errors in understanding the agent's intentions. In application such as autonomous driving, errors in the reward function can cause dangerous situations. For this reason, the verification through simulated environment of the effectiveness of the retrieve rewards is quite important. 
\bibliographystyle{plain}
\bibliography{biblio.bib}

\begin{thebibliography}{10}

\bibitem{abbeel2007application}
Pieter Abbeel, Adam Coates, Morgan Quigley, and Andrew~Y Ng.
\newblock An application of reinforcement learning to aerobatic helicopter
  flight.
\newblock In {\em Advances in neural information processing systems}, pages
  1--8, 2007.

\bibitem{abbeel2004apprenticeship}
Pieter Abbeel and Andrew~Y Ng.
\newblock Apprenticeship learning via inverse reinforcement learning.
\newblock In {\em Proceedings of the twenty-first international conference on
  Machine learning}, page~1, 2004.

\bibitem{argall2009survey}
Brenna~D Argall, Sonia Chernova, Manuela Veloso, and Brett Browning.
\newblock A survey of robot learning from demonstration.
\newblock {\em Robotics and autonomous systems}, 57(5):469--483, 2009.

\bibitem{balakrishna2019policy}
Ashwin Balakrishna, Brijen Thananjeyan, Jonathan Lee, Felix Li, Arsh Zahed,
  Joseph~E. Gonzalez, and Ken Goldberg.
\newblock On-policy robot imitation learning from a converging supervisor.
\newblock In Leslie~Pack Kaelbling, Danica Kragic, and Komei Sugiura, editors,
  {\em 3rd Annual Conference on Robot Learning, CoRL 2019, Osaka, Japan,
  October 30 - November 1, 2019, Proceedings}, volume 100 of {\em Proceedings
  of Machine Learning Research}, pages 24--41. {PMLR}, 2019.

\bibitem{baxter2001infinite}
Jonathan Baxter and Peter~L Bartlett.
\newblock Infinite-horizon policy-gradient estimation.
\newblock {\em Journal of Artificial Intelligence Research}, 15:319--350, 2001.

\bibitem{DBLP:journals/corr/BrockmanCPSSTZ16}
Greg Brockman, Vicki Cheung, Ludwig Pettersson, Jonas Schneider, John Schulman,
  Jie Tang, and Wojciech Zaremba.
\newblock Openai gym.
\newblock {\em CoRR}, abs/1606.01540, 2016.

\bibitem{brown2019extrapolating}
Daniel Brown, Wonjoon Goo, Prabhat Nagarajan, and Scott Niekum.
\newblock Extrapolating beyond suboptimal demonstrations via inverse
  reinforcement learning from observations.
\newblock In {\em International Conference on Machine Learning}, pages
  783--792, 2019.

\bibitem{casella2002statistical}
George Casella and Roger~L Berger.
\newblock {\em Statistical inference}, volume~2.
\newblock Duxbury Pacific Grove, CA, 2002.

\bibitem{castro2019inverse}
Pablo~Samuel Castro, Shijian Li, and Daqing Zhang.
\newblock Inverse reinforcement learning with multiple ranked experts.
\newblock {\em arXiv preprint arXiv:1907.13411}, 2019.

\bibitem{chen2013noisy}
Yudong Chen and Constantine Caramanis.
\newblock Noisy and missing data regression: Distribution-oblivious support
  recovery.
\newblock In {\em International Conference on Machine Learning}, pages
  383--391, 2013.

\bibitem{christiano2017deep}
Paul~F Christiano, Jan Leike, Tom Brown, Miljan Martic, Shane Legg, and Dario
  Amodei.
\newblock Deep reinforcement learning from human preferences.
\newblock In {\em Advances in Neural Information Processing Systems}, pages
  4299--4307, 2017.

\bibitem{chung2010mobile}
Shu-Yun Chung and Han-Pang Huang.
\newblock A mobile robot that understands pedestrian spatial behaviors.
\newblock In {\em 2010 IEEE/RSJ International Conference on Intelligent Robots
  and Systems}, pages 5861--5866. IEEE, 2010.

\bibitem{goyal2019first}
Vineet Goyal and Julien Grand-Clement.
\newblock A first-order approach to accelerated value iteration.
\newblock {\em arXiv preprint arXiv:1905.09963}, 2019.

\bibitem{haarnoja2018soft}
Tuomas Haarnoja, Aurick Zhou, Pieter Abbeel, and Sergey Levine.
\newblock Soft actor-critic: Off-policy maximum entropy deep reinforcement
  learning with a stochastic actor.
\newblock {\em arXiv preprint arXiv:1801.01290}, 2018.

\bibitem{hussein2017imitation}
Ahmed Hussein, Mohamed~Medhat Gaber, Eyad Elyan, and Chrisina Jayne.
\newblock Imitation learning: A survey of learning methods.
\newblock {\em ACM Computing Surveys (CSUR)}, 50(2):21, 2017.

\bibitem{ibarz2018reward}
Borja Ibarz, Jan Leike, Tobias Pohlen, Geoffrey Irving, Shane Legg, and Dario
  Amodei.
\newblock Reward learning from human preferences and demonstrations in atari.
\newblock In {\em Advances in Neural Information Processing Systems}, pages
  8011--8023, 2018.

\bibitem{lflpaper}
Alexis Jacq, Matthieu Geist, Ana Paiva, and Olivier Pietquin.
\newblock Learning from a learner.
\newblock In Kamalika Chaudhuri and Ruslan Salakhutdinov, editors, {\em
  Proceedings of the 36th International Conference on Machine Learning},
  volume~97 of {\em Proceedings of Machine Learning Research}, pages
  2990--2999, Long Beach, California, USA, 09--15 Jun 2019. PMLR.

\bibitem{mcwilliams2014fast}
Brian McWilliams, Gabriel Krummenacher, Mario Lucic, and Joachim~M Buhmann.
\newblock Fast and robust least squares estimation in corrupted linear models.
\newblock In {\em Advances in Neural Information Processing Systems}, pages
  415--423, 2014.

\bibitem{metelli2017compatible}
Alberto~Maria Metelli, Matteo Pirotta, and Marcello Restelli.
\newblock Compatible reward inverse reinforcement learning.
\newblock In {\em Advances in Neural Information Processing Systems}, pages
  2050--2059, 2017.

\bibitem{ng2000algorithms}
Andrew~Y. Ng and Stuart~J. Russell.
\newblock Algorithms for inverse reinforcement learning.
\newblock In {\em Proceedings of the Seventeenth International Conference on
  Machine Learning {(ICML} 2000), Stanford University, Stanford, CA, USA, June
  29 - July 2, 2000}, pages 663--670. Morgan Kaufmann, 2000.

\bibitem{DBLP:journals/ftrob/OsaPNBA018}
Takayuki Osa, Joni Pajarinen, Gerhard Neumann, J.~Andrew Bagnell, Pieter
  Abbeel, and Jan Peters.
\newblock An algorithmic perspective on imitation learning.
\newblock {\em Foundations and Trends in Robotics}, 7(1-2):1--179, 2018.

\bibitem{peters2006policy}
Jan Peters and Stefan Schaal.
\newblock Policy gradient methods for robotics.
\newblock In {\em 2006 IEEE/RSJ International Conference on Intelligent Robots
  and Systems}, pages 2219--2225. IEEE, 2006.

\bibitem{peters2008reinforcement}
Jan Peters and Stefan Schaal.
\newblock Reinforcement learning of motor skills with policy gradients.
\newblock {\em Neural networks}, 21(4):682--697, 2008.

\bibitem{pirotta2016inverse}
Matteo Pirotta and Marcello Restelli.
\newblock Inverse reinforcement learning through policy gradient minimization.
\newblock In Dale Schuurmans and Michael~P. Wellman, editors, {\em Proceedings
  of the Thirtieth {AAAI} Conference on Artificial Intelligence, February
  12-17, 2016, Phoenix, Arizona, {USA}}, pages 1993--1999. {AAAI} Press, 2016.

\bibitem{pirotta2013adaptive}
Matteo Pirotta, Marcello Restelli, and Luca Bascetta.
\newblock Adaptive step-size for policy gradient methods.
\newblock In {\em Advances in Neural Information Processing Systems}, pages
  1394--1402, 2013.

\bibitem{puterman1994markov}
Martin~L. Puterman.
\newblock {\em Markov Decision Processes: Discrete Stochastic Dynamic
  Programming}.
\newblock John Wiley \& Sons, Inc., New York, NY, USA, 1994.

\bibitem{rabinowitz2018machine}
Neil Rabinowitz, Frank Perbet, Francis Song, Chiyuan Zhang, SM~Ali Eslami, and
  Matthew Botvinick.
\newblock Machine theory of mind.
\newblock In {\em International Conference on Machine Learning}, pages
  4218--4227, 2018.

\bibitem{fixeddesign}
Philippe Rigollet.
\newblock High-dimensional statistics. spring 2015.
\newblock {\em Massachusetts Institute of Technology: MIT OpenCourseWare},
  2015.

\bibitem{schulman2017equivalence}
John Schulman, Xi~Chen, and Pieter Abbeel.
\newblock Equivalence between policy gradients and soft q-learning.
\newblock {\em arXiv preprint arXiv:1704.06440}, 2017.

\bibitem{schulman2017proximal}
John Schulman, Filip Wolski, Prafulla Dhariwal, Alec Radford, and Oleg Klimov.
\newblock Proximal policy optimization algorithms.
\newblock {\em arXiv preprint arXiv:1707.06347}, 2017.

\bibitem{shteingart2014reinforcement}
Hanan Shteingart and Yonatan Loewenstein.
\newblock Reinforcement learning and human behavior.
\newblock {\em Current Opinion in Neurobiology}, 25:93--98, 2014.

\bibitem{shum2019theory}
Michael Shum, Max Kleiman-Weiner, Michael~L Littman, and Joshua~B Tenenbaum.
\newblock Theory of minds: Understanding behavior in groups through inverse
  planning.
\newblock In {\em Proceedings of the AAAI Conference on Artificial
  Intelligence}, volume~33, pages 6163--6170, 2019.

\bibitem{spokoiny2012parametric}
Vladimir Spokoiny et~al.
\newblock Parametric estimation. finite sample theory.
\newblock {\em The Annals of Statistics}, 40(6):2877--2909, 2012.

\bibitem{sutton1998reinforcement}
Richard~S. Sutton and Andrew~G. Barto.
\newblock {\em Reinforcement Learning: An Introduction}.
\newblock A Bradford book. Bradford Book, 1998.

\bibitem{sutton2000policy}
Richard~S Sutton, David~A McAllester, Satinder~P Singh, and Yishay Mansour.
\newblock Policy gradient methods for reinforcement learning with function
  approximation.
\newblock In {\em Advances in Neural Information Processing Systems}, pages
  1057--1063, 2000.

\bibitem{tseng2001convergence}
Paul Tseng.
\newblock Convergence of a block coordinate descent method for
  nondifferentiable minimization.
\newblock {\em Journal of optimization theory and applications},
  109(3):475--494, 2001.

\bibitem{vogel2012improving}
Adam Vogel, Deepak Ramachandran, Rakesh Gupta, and Antoine Raux.
\newblock Improving hybrid vehicle fuel efficiency using inverse reinforcement
  learning.
\newblock In {\em Twenty-Sixth AAAI Conference on Artificial Intelligence},
  2012.

\bibitem{wedin1973perturbation}
Perke Wedin.
\newblock Perturbation theory for pseudo-inverses.
\newblock {\em BIT Numerical Mathematics}, 13(2):217--232, 1973.

\bibitem{williams1992simple}
Ronald~J Williams.
\newblock Simple statistical gradient-following algorithms for connectionist
  reinforcement learning.
\newblock {\em Machine learning}, 8(3-4):229--256, 1992.

\bibitem{ziebart2008maximum}
Brian~D. Ziebart, Andrew~L. Maas, J.~Andrew Bagnell, and Anind~K. Dey.
\newblock Maximum entropy inverse reinforcement learning.
\newblock In Dieter Fox and Carla~P. Gomes, editors, {\em Proceedings of the
  Twenty-Third {AAAI} Conference on Artificial Intelligence, {AAAI} 2008,
  Chicago, Illinois, USA, July 13-17, 2008}, pages 1433--1438. {AAAI} Press,
  2008.

\bibitem{ziebart2008navigate}
Brian~D Ziebart, Andrew~L Maas, Anind~K Dey, and J~Andrew Bagnell.
\newblock Navigate like a cabbie: Probabilistic reasoning from observed
  context-aware behavior.
\newblock In {\em Proceedings of the 10th international conference on
  Ubiquitous computing}, pages 322--331, 2008.

\end{thebibliography}
\newpage
\appendix
\section{Proofs and derivations}
\label{apx:proofs}
We start the proofs given some introduction on Pertubation on Least Square problems and on Least Square problems with fixed design. Then we report the proofs and derivations for the results of Sections~\ref{sec:exact} and~\ref{sec:learndem}. For the rest of this section we assume that:
\begin{itemize}
    \item $\vomega^L, \widehat{\vomega} \in \reals^q$,
    \item $\vtheta^L, \widehat{\vtheta} \in \reals^d$.
\end{itemize}
We define with $\vomega^L$ ($\vtheta^L$) the reward (policy) parameters of the learner, and with $\widehat{\vomega}$ ($\widehat{\vtheta}$) the reward (policy) parameters recovered by the observer.

\subsection{Preliminaries}
\begin{definition}
The condition number of a matrix $\mathbr{A} \in R^{m \times q}$ $\mathbr{A} \ne 0$ is:
\begin{equation*}
    \kappa = \norm{\mathbr{A}}_2 \norm{\mathbr{A}^{+}}_2 = \frac{\sigma_1}{\sigma_r},
\end{equation*}
where 0 < r = rank($\mathbr{A}$) $\le \min(m,q)$, and $\sigma_1 \ge \cdots \ge \sigma_r > 0$ are the nonzero singular values of $\mathbr{A}$.
\end{definition}

A least-squares problem is defined as:
\begin{equation}
\label{leastsquare}
    \min_{x} \norm{\mathbr{A}x - b}_2,
\end{equation}
where the solution is $x = \mathbr{A}^{+} b$. We denote with $\mathbr{A}^{+}$ the pseudoinverse of $\mathbr{A}$, the perturbed $\mathbr{A}$ as $\widehat{\mathbr{A}} = \mathbr{A} + \delta \mathbr{A}$ and the pertubed $\hat{b} = b + \delta b$ and the perturbed solution $\widehat{x} = \widehat{\mathbr{A}}^{+} \hat b = x + \delta x$. Finally, we denote with $\mathbr{A}^H$ the adjoint of the matrix $\mathbr{A}$.

We define as $\chi = \frac{\norm{\delta \mathbr{A}}_2}{\norm{\mathbr{A}}_2}$ and $y = \mathbr{A}^{+H}x$.
\begin{lemma}[Perturbation on Least Square Problems \cite{wedin1973perturbation}]
\label{lemma:perturbation}
Assume that rank($\mathbr{A}+\delta \mathbr{A}$) = rank($\mathbr{A}$) and $\chi \kappa < 1$ then:
\begin{equation}
    \norm{x-\hat{x}}_2 \le \frac{\kappa}{(1-\chi\kappa) \norm{\mathbr{A}}_2 } (\chi \norm{x}_2 \norm{\mathbr{A}}_2 + \chi \kappa \norm{r}_2 + \norm{\delta b}_2) + \chi \norm{y}_2 \norm{\mathbr{A}}_2.
\end{equation}
\end{lemma}
\begin{proof}
The proof can be find in \cite{wedin1973perturbation}.
\end{proof}

We adapt the lemma 6 in \cite{chen2013noisy} to our context where $\widehat{\vomega}$ are the reward weights recovered with lemma \ref{lemma:weightclosedformsum}.
\begin{lemma}[From lemma 6 in \cite{chen2013noisy}]
\label{boundold}
Let $\Sigma = (\estimatedjacobian^T \estimatedjacobian)$ and suppose the following strong convexity condition holds: $\lambda_{\min}(\Sigma) \ge \lambda > 0$. Then the estimation error satisfies:
\begin{equation*}
    \norm{\widehat{\vomega} - \vomega^L}_2 \le O\left(\frac{1}{\lambda} \norm{\jacobian^T \Delta - \Sigma \vomega^L}_2\right).
\end{equation*}
\end{lemma}

\begin{lemma}[Revised from lemma 11 in \cite{mcwilliams2014fast}]
\label{subgaus}
Suppose $X \in \reals^{m \times q}$ and $W \in \reals^{n \times M}$ are zero-mean sub-gaussian matrices with parameters $(\frac{1}{n} \Sigma_x, \frac{1}{n} \sigma^2_x), (\frac{1}{n} \Sigma_w, \frac{1}{n} \sigma^2_w)$ respectively. Then for any fixed vectors $v_1, v_2$, we have:
\begin{equation*}
    \text{P}[|v_1^T (W^T X - \EX[W^TW]) v_2 | \ge t \norm{v_1}_2 \norm{v_2}_2] \le 3 exp \left(-cn \min \left\{ \frac{t^2}{\sigma_x^2 \sigma_w^2},  \frac{t}{\sigma_x\sigma_w} \right\}\right), 
\end{equation*}
in particular if $n \gtrsim \log p$ we have that:
\begin{equation*}
    |v_1^T (W^T X - \EX[W^TW]) v_2 | \le \sigma_x \sigma_w \norm{v_1}_2 \norm{v_2}_2 \sqrt{\frac{\log p}{n}}.
\end{equation*}
Setting $v_1$ to be the first standard basis vector and using a union bound over $j = 1, \cdots, p$ we have:
\begin{equation*}
    \norm{(W^T X - \EX[W^TX])v}_\infty \le \sigma_x \sigma_w \norm{v}_2 \sqrt{\frac{\log p}{n}},
\end{equation*}
with probability $1 - c_1 exp(-c_2 \log p)$ where $c_1, c_2$ are positive constants which are independent from $\sigma_x, \sigma_w, n, p$.
\end{lemma}

\begin{theorem}[from Chapter 2 \cite{fixeddesign}]
\label{fixeddesign}
Assume that the least-squares model:
\begin{equation*}
    \min_{x} \norm{\mathbr{A}x - b + \epsilon}
\end{equation*}
holds where $\epsilon \sim \text{subGn}(\sigma^2)$. Then, for any $\delta > 0$, with probability $1-\delta$ it holds:
\begin{equation*}
    \norm{x - \hat{x}}_2 \le \sigma \sqrt{\frac{r + \log(\frac{1}{\delta})}{n \sigma_{\min}}},
\end{equation*}
where $\sigma_{\min} = \frac{\mathbr{A}^T\mathbr{A}}{n}$ is the minimum singular value of $\mathbr{A}^T\mathbr{A}$ and $r$ is the rank($\mathbr{A}^T\mathbr{A}$).
\end{theorem}

\subsection{Proofs and derivation of Section~\ref{sec:exact}}
In this section we give the proofs and derivations of the theorems in Section~\ref{sec:exact}.

First, we will provide a finite sample analysis on the difference in norm between the reward vector of the learner $\vomega^L$ and the reward vector recoverd using~\eqref{eq:omega}, with a single learning step. This result was omitted in the main paper as we can see this as a special case of Theorem~\ref{finitesampletwo}, but with a different technique. We add it here as it provides a first insight on how, having enough demonstrations, we can recover the correct weights. In the demonstration, without loss of generality, we assume that the learning rate is $1$.
\begin{lemma}
\label{lemma:boundgradient}
Let $\jacobian$ be the real Jacobian and $\estimatedjacobian$ the estimated Jacobian from $n$ trajectories $\{\tau_1, \cdots, \tau_n\}$. Assume that $\widehat{\nabla_{\vtheta} \psi}$ is uniformly bounded by $M$. Then with probability $1- \delta$
\begin{equation*}
    \norm{\widehat{\nabla_{\vtheta} \psi} -\nabla_{\vtheta} \psi}_2 \le M \sqrt{qd} \sqrt{\frac{\log(\frac{2}{\delta})}{2n}}.
\end{equation*}
\end{lemma}
\begin{proof}
We use Hoeffding's inequality:
\begin{align*}
    \text{P}\left[\norm{\widehat{\nabla_{\vtheta} \psi} -\nabla_{\vtheta} \psi}_2 \ge t\right] 
   \le  \text{P}\left[\sqrt{qd}\norm{\widehat{\nabla_{\vtheta} \psi} -\nabla_{\vtheta} \psi}_{\infty} \ge t\right]
    \le 2 \exp \left(\frac{-2t^2n}{dqM^2} \right)
\end{align*}
The result follows by setting $\delta =  2 \exp \left(\frac{-2t^2n}{dqM^2} \right)$.
\end{proof}

\begin{restatable}[]{thm}{finitesampleone}
\label{}
Let $\nabla_{\vtheta} \psi$ be the real Jacobian and $\widehat{\nabla_{\vtheta} \psi}$ the estimated Jacobian from $n$ trajectories $\{\tau_1, \cdots, \tau_n\}$. Assume that $\widehat{\nabla_{\vtheta} \psi}$ is uniformly bounded by $M$, $rank(\estimatedjacobian)=rank(\jacobian)$ and $\norm{\estimatedjacobian - \jacobian}_2 \cdot\kappa_{\jacobian} < \norm{\jacobian}_2$. Then with probability $1 - \delta$:
\begin{equation}
    \norm{\vomega^L - \hat{\vomega}}_2 \le M \sqrt{qd} \sqrt{\frac{\log(\frac{2}{\delta})}{2n}} \left (\frac{\kappa_{\jacobian} \norm{\vomega^L}_2 }{c \norm{\jacobian}_2} + \norm{y}_2 \right),
\end{equation}
where $\omega^L$ are the real reward parameters and $\hat{\omega}$ are the parameters recovered with Equation~\eqref{eq:omega}, $c = 1-\frac{\norm{\estimatedjacobian - \jacobian}_2}{\norm{\jacobian}_2} \kappa_{\jacobian} > 0$, and $y = \jacobian^{+H}\vomega$. 
\end{restatable}\begin{proof}
We need to bound the difference in norm between $\vomega^L$ and $\widehat{\vomega}$ that are the true parameters and the parameters that we recovered solving the minimization problem~\eqref{min:sumomegas}. 
 \resizebox{\linewidth}{!}{%
 \begin{minipage}{\linewidth}
\begin{align}
    \norm{\vomega^L - \widehat{\vomega}}_2 \\
    \le \frac{\kappa}{\left(1-\kappa \frac{\norm{\delta \jacobian}_2}{\norm{\jacobian}_2}\right) \norm{\jacobian}_2 } \left(\ \frac{\norm{\delta \jacobian}_2}{\norm{\jacobian}_2} \norm{\vomega^L}_2 \norm{\jacobian}_2 \right) +  \frac{\norm{\delta \jacobian}_2}{\norm{\jacobian}_2} \norm{y}_2 \norm{\jacobian}_2  \\ 
    \le \frac{\kappa}{c\norm{\jacobian}_2}\left(\ \frac{\norm{\delta \jacobian}_2}{\norm{\jacobian}_2} \norm{\vomega^L}_2 \norm{\jacobian}_2\right) +  \frac{\norm{\delta \jacobian}_2}{\norm{\jacobian}_2} \norm{y}_2 \norm{\jacobian}_2\\ 
    = \norm{\delta \jacobian}_2 \left(\frac{\kappa \norm{\vomega^L}_2 }{c\norm{\jacobian}_2} + \norm{y}_2\right) \\ 
    \le M \sqrt{qd} \sqrt{\frac{\log(\frac{2}{\delta})}{2n}} \left(\frac{\kappa \norm{\vomega^L}_2 }{c\norm{\jacobian}_2} + \norm{y}_2\right),
\end{align}
\end{minipage}}
where line (14) is obtained by using Lemma~\ref{lemma:perturbation}, lines (15, 16) by rearranging the terms, and line (17) by using Lemma~\ref{lemma:boundgradient}. We can observe that the last term vanishes when the rank($\jacobian$)$=q$ (see \cite{wedin1973perturbation}).
\end{proof}

Now we will give the proofs and derivations of Lemmas~\ref{lemma:weightclosedformsum}  and Theorem~\ref{finitesampletwo}.



\weightclosedformsum*
\begin{proof}
Taking the derivative of~\eqref{min:sumomegas} with respect to $\omega$:
\begin{align*}
    \nabla_{\vomega} \sum_{t=1}^{m} \norm{\Delta_t - \alpha_t \nabla_{\vtheta} \vpsi_t\vomega}_2^2 & = 
      \sum_{t=1}^m \nabla_{\vomega} (\Delta_t - \alpha_t \nabla_{\vtheta} \vpsi_t \vomega)^T (\Delta - \alpha \nabla_{\vtheta} \vpsi_t \vomega)\\
      &=\sum_{t=1}^m \nabla_{\vomega} (\Delta_t ^T \Delta_t + (\alpha \nabla_{\vtheta} \vpsi_t \vomega)^T (\alpha_t \nabla_{\vtheta} \vpsi_t \vomega) - 2 \alpha_t \nabla_{\vtheta} \vpsi_t \vomega)^T \Delta_t) \\
     &=2 \left (\sum_{t=1}^m \alpha_t^2 \nabla_{\vtheta} \vpsi_t ^T \nabla_{\vtheta} \vpsi_t \right ) \vomega  - 2 \sum_{t=1}^m \left (\alpha_t \nabla_{\vtheta} \vpsi_t ^T \Delta_t \right ).
\end{align*}
Taking it equal to zero:
\begin{align*}
     \left (\sum_{t=1}^m \alpha_t^2 \nabla_{\vtheta} \vpsi_t ^T \nabla_{\vtheta} \vpsi_t \right ) \vomega  - \sum_{t=1}^m \left (\alpha_t \nabla_{\vtheta} \vpsi_t ^T \Delta_t \right ) = 0 \\
     \vomega = \left (\sum_{t=1}^m \alpha_t^2 \nabla_{\vtheta} \vpsi_t ^T \nabla_{\vtheta} \vpsi_t \right )^{-1} \left (\sum_{t=1}^m \alpha_t \nabla_{\vtheta} \vpsi_t ^T \Delta_t \right )
\end{align*}
\end{proof}
\begin{lemma}
\label{lemma:regularized}
The regularized version of~\eqref{min:sumomegas} is equal to:
\begin{equation*}
    \min_{\vomega} \sum_{t=1}^m \norm{\Delta_t - \alpha_t \jacobian_t \vomega}^2_2 + \lambda \norm{\vomega}^2_2,
\end{equation*}
where $\lambda > 0$.
We can solve the regularized problem in closed form:
\begin{equation*}
    \vomega = \left (\sum_{t=1}^m \alpha_t^2 \nabla_{\vtheta} \vpsi_t ^T \nabla_{\vtheta} \vpsi_t + \lambda \mathbf{I}_d\right )^{-1} \left (\sum_{t=1}^m \alpha_t \nabla_{\vtheta} \vpsi_t ^T \Delta_t \right ).
\end{equation*}
\end{lemma}
\begin{proof}
Taking the derivative respect to $\omega$:
\begin{align*}
    \nabla_{\vomega} \sum_{t=1}^{m} \norm{\Delta_t - \alpha_t \nabla_{\vtheta} \vpsi_t\vomega}_2^2 + \lambda\norm{\vomega}_2^2  =
      \sum_{t=1}^m \nabla_{\vomega} (\Delta_t - \alpha_t \nabla_{\vtheta} \vpsi_t \vomega)^T (\Delta - \alpha \nabla_{\vtheta} \vpsi_t \vomega) + \nabla_{\vomega} \lambda \vomega^T \vomega\\
      =\sum_{t=1}^m \nabla_{\vomega} (\Delta_t ^T \Delta_t + (\alpha \nabla_{\vtheta} \vpsi_t \vomega)^T (\alpha_t \nabla_{\vtheta} \vpsi_t \vomega) - 2 \alpha_t \nabla_{\vtheta} \vpsi_t \vomega)^T \Delta_t) + 2\lambda\vomega \\
     =2 \left (\sum_{t=1}^m \alpha_t^2 \nabla_{\vtheta} \vpsi_t ^T \nabla_{\vtheta} \vpsi_t \right ) \vomega  - 2 \sum_{t=1}^m \left (\alpha_t \nabla_{\vtheta} \vpsi_t ^T \Delta_t \right ) + 2\lambda\vomega.
\end{align*}
Taking it equal to zero:
\begin{align*}
     \left (\sum_{t=1}^m \alpha_t^2 \nabla_{\vtheta} \vpsi_t ^T \nabla_{\vtheta} \vpsi_t + \lambda \mathbf{I}_d\right) \vomega  - \sum_{t=1}^m \left (\alpha_t \nabla_{\vtheta} \vpsi_t ^T \Delta_t \right ) = 0 \\
     \vomega = \left (\sum_{t=1}^m \alpha_t^2 \nabla_{\vtheta} \vpsi_t ^T \nabla_{\vtheta} \vpsi_t + \lambda \mathbf{I}_d \right )^{-1} \left (\sum_{t=1}^m \alpha_t \nabla_{\vtheta} \vpsi_t ^T \Delta_t \right ).
\end{align*}
\end{proof}
\finitesampletwo*
\begin{proof}
We decompose the estimated Jacobian $\grads = \grads + E$, where $E$ is the random variable component caused by the estimation of the $\jacobian$. Since we estimate the jacobians with an unbiased estimator the mean of $E$ is $0$. We reshape the $\grads$ and $E$ as $\grads \in \reals^{m \times dq}$ and $E \in \reals^{m \times dq}$. Now $E$, since its mean is $0$ and all lines are independent of each other, is a sub-Gaussian matrix with parameters $(\frac{1}{m} \Sigma_E, \frac{1}{m}\sigma_E)$. The proof is similar to the proof of Theorem 1 in~\cite{mcwilliams2014fast}.
\begin{align*}
     \norm{  (\grads + E)^T(\grads^T \vomega^L) - (\grads + E)^T(\grads + E)\vomega^L}_2 \\
     = \norm{  \grads^T \jacobian^T \vomega^* + E^T \grads^T \vomega^L - \grads^T \grads \vomega^L - \grads^T E \vomega^* - E \grads^T \vomega^* - E^T E \vomega^L}_2 \\
     = \norm{ - \grads^T E \vomega^L - E^T E \vomega^L}_2.
\end{align*}
Now we bound separately these two terms, using Lemma~\ref{subgaus} as in \cite{mcwilliams2014fast}:
\begin{align*}
    \norm{\grads^T E \vomega}_2 &\le \norm{\grads}_2 \sigma_E \norm{\vomega^L}_2 \sqrt{\frac{\log dq}{m}} \\
    \norm{E^TE\vomega^L}_2 &= \norm{(E^TE + \sigma_E^2 \identity_{qd} - \sigma_E^2 \identity_{qd} )\vomega^L}_2\le \sigma_E^2 \left(C \sqrt{\frac{\log dq}{m}} + \sqrt{dq}\right) \norm{\vomega^L}_2
\end{align*}
with probability $1 - c_1 \exp(-c_2 \log q)$ where $c_1, c_2$ are positive constants that do not depend on $\sigma_E, n, q$.
So now applying Lemma~\ref{boundold}:
\begin{equation*}
    \norm{\vomega^L - \widehat{\vomega}}_2 \le \frac{1}{\lambda}\left(\norm{\grads}_2 \sigma_E \norm{\vomega^L}_2 \sqrt{\frac{\log dq}{m}} + \sigma_E^2 \left(C \sqrt{\frac{\log dq}{m}} + \sqrt{dq}\right) \norm{\vomega^L}_2\right)
\end{equation*}
w.h.p..

Now we need to bound the random variable $\sigma_E$. Remember that $E_i=\jacobian_i - \estimatedjacobian_i$. Since $\estimatedjacobian$ are assumed to be bounded by $M$, by applying Hoeffding's inequality, with probability $1-\delta_1$:
\begin{equation*}
    \norm{E_i}_2 = \norm{\estimatedjacobian_i - \jacobian_i}_2 \le M  \sqrt{\frac{dq\log(\frac{2}{\delta_1})}{2n}}. 
\end{equation*}
So $E$ is a subgaussian random variable where each component is bounded by $M  \sqrt{\frac{dq\log(\frac{2}{\delta_1})}{2n}}$. 

Then:
\begin{align*}
\text{P}\left[ \norm{\vomega^L - \widehat{\vomega}}_2 \ge  \frac{1}{\lambda} M \sqrt{\frac{dq\log(\frac{2}{\delta_1})}{2n}}   \norm{\vomega^L}_2 \norm{\grads}_2 \sqrt{\frac{\log dq}{m}} +  M \frac{dq\log(\frac{2}{\delta_1})}{2n} C \sqrt{\frac{\log dq}{m}} + \sqrt{dq}\right] \\ \le 
    \text{P}\left[ \norm{\vomega^L - \widehat{\vomega}}_2 \ge  \frac{1}{\lambda} M \sqrt{\frac{dq\log(\frac{2}{\delta_1})}{2n}}   \norm{\vomega^L}_2 \left (\norm{\grads}_2 \sqrt{\frac{\log dq}{m}} + C \sqrt{\frac{\log dq}{m}} + \sqrt{dq}\right)\right]\\
    \le \delta_1 + c1 \exp (-c2 \log dq)
\end{align*}
So the result follows where with w.h.p. we mean with probability $1-(\delta_1 + c1 \exp (-c2 \log dq))$ as in~\cite{mcwilliams2014fast}.
\end{proof}

\subsection{Proofs of Section~\ref{sec:learndem}}
In this section we provide the proofs and derivations of the theorems in Section~\ref{sec:learndem}.

\begin{lemma}
\label{lemma:MLE}
Given a dataset $D = \{(s_1,a_1), \cdots, (s_n,a_m)\}$ of state-action couples sampled from a Gaussian linear policy $\pi_{\vtheta}(\cdot|s) \sim \mathcal{N}(\vtheta^T\vvarphi(s), \sigma^2)$ such that $S \in \reals^{n \times p}$ is the matrix of states features and let the minimum singular value of $(S^TS)$ $\sigma_{\min} \ge \eta$, then the error between the maximum likelihood estimator $\vtheta^{MLE}$ and the mean $\vtheta$ is, with probability $1- \delta$:
\begin{equation*}
    \norm{\vtheta^{\text{MLE}} - \vtheta}_2 \le \sigma \sqrt{\frac{r + \log(\frac{1}{\delta})}{n \eta}},
\end{equation*}
where $r$ is the rank($S^TS$).
\end{lemma}
\label{lemmamle}
\begin{proof}
We start by stating that the maximum likelihood for linear Gaussian policies can be recast as an ordinary least-squares problem.  We write the Likelihood $L(\vtheta)$
\begin{align*}
\log L(\vtheta) = \log \left( \prod_{i=1}^n\pi(a_i|s_i)  \right) \\
= \sum_{i=1}^n \log \left(\frac{1}{\sqrt{2\pi\sigma^2}} \exp \left(-\frac{(a_i-\vtheta^T \vvarphi(s_i))^2}{2\sigma^2}\right)\right) \\
= n \log \left(\frac{1}{\sqrt{2\pi\sigma^2}}\right) - \sum_{i=1}^n \frac{(a_i-\vtheta^T \vvarphi(s_i))^2}{2\sigma^2}
\end{align*}
The resulting maximum likelihood problem is given by:
\begin{equation*}
     \max_{\vtheta} \log L(\vtheta) = \min_{\vtheta} \sum_{i=1}^n (a_i-\vtheta^T \vvarphi(s_i))^2
\end{equation*}
So we have the following linear least-squares problem:
\begin{equation*}
    \min_{\theta} \norm{S\vtheta - A + \epsilon}_2,
\end{equation*}
where $\epsilon$ is an error with mean $0$ and variance $\sigma^2$, $S \in \reals^{n\times d}$ is the matrix of states features and $A \in \reals^n$ is the vector of actions. Using Theorem~\ref{fixeddesign}, we can say that with probability $1- \delta$:
\begin{align*}
    \norm{\vtheta^{\text{MLE}} - \vtheta}_2 \le \sigma \sqrt{\frac{r + \log(\frac{1}{\delta})}{n \eta}},
\end{align*}
where $r$ is the rank($S^TS$).
\end{proof}

\begin{lemma}
Given two Gaussian policies $\pi_{\vtheta_1}(\cdot|s) \sim \mathcal{N}(\vtheta_1^T\vvarphi(s), \sigma^2)$ and $\pi_{\vtheta_2}(\cdot|s) \sim \mathcal{N}(\vtheta_2^T\vvarphi(s), \sigma^2)$ with same variance and  the state features are bounded by $M_S$:
\begin{equation*}
    \norm{\nabla_{\vtheta} \log \pi_{\vtheta_1}(a|s) - \nabla_{\vtheta} \log \pi_{\vtheta_2}(a|s)}_2 \le \frac{M_{S}^2}{\sigma^2} \norm{\vtheta_1 - \vtheta_2}_2.
\end{equation*}
\end{lemma}
\label{gaussianpolicies}
\begin{proof}
The gradient of the log policy of a general policy $\pi_{\vtheta}(a|s)$ is:
\begin{equation*}
    \nabla_{\vtheta} \log \pi(a|s) = \frac{\vvarphi(s)^T (a-\vtheta^T \vvarphi(s))}{\sigma^2}.
\end{equation*}
Now we apply this result to the difference in norm between two Gaussian log policies:
\begin{align}
    \norm{\nabla_{\vtheta} \log \pi_{\vtheta_1}(a|s) - \nabla_{\vtheta} \log \pi_(a|s)}_2 
    &= \norm{\frac{\vvarphi(s)^T (a-\vtheta_1^T \vvarphi(s))}{\sigma^2} - \frac{\vvarphi(s)^T (a-\vtheta_2^T \vvarphi(s))}{\sigma^2}}_2 \\
    &= \norm{\frac{\vvarphi(s)}{\sigma^2} (\vtheta_1^T \vvarphi(s) -\vtheta_2^T \vvarphi(s))}_2 \\
    &\le \norm{\frac{\vvarphi(s)}{\sigma^2}}_2 \norm{\vtheta_1 -\vtheta_2}_2 \norm{\vvarphi(s)}_2 \\
    &\le \frac{M_S^2}{\sigma^2} \norm{\vtheta_1 -\vtheta_2}_2.
\end{align}
In line (19) we use the Cauchy-Schwartz inequality,  and in line (20) the assumption that the state features are bounded by $M_S$. 
\end{proof}

\begin{lemma}\label{differencesgrad}
Given a dataset $D = \{\tau_1, \cdots, \tau_n\}$ of trajectories such that every trajectory $\tau_i = \{(s_1,a_1), \cdots, (s_T,a_T)\}$ is sampled from a Gaussian linear policy $\pi_{\vtheta}(\cdot|s) \sim \mathcal{N}(\vtheta^T\vvarphi(s), \sigma)$, the maximum likelihood estimator $\vtheta^{MLE}$ estimated on $D$, the condition of Lemma~\ref{lemma:MLE} holds, the $\estimatedjacobian$ uniformly bounded by $M$, the state features bounded by $M_{S}$, the reward features bounded by $M_R$. Let $S \in \reals^{n \times p}$ be the matrix of state features and let $\sigma_{\min}(S^TS) \ge \eta$. Then with probability $1-\delta$:
\begin{equation*}
    \norm{\estimatedjacobian(\vtheta^{MLE}) - \jacobian(\vtheta)}_2 \le
    M \sqrt{qd} \sqrt{\frac{\log(\frac{2}{\delta})}{2n}}  + \frac{T M_{S}^2 M_R}{(1-\gamma) \sigma}  \sqrt{\frac{r + \log(\frac{1}{\delta})}{n \eta}},
\end{equation*}
where $\gamma$ is the discount factor and $r$ is the rank of $S^TS$. 
\end{lemma}
\begin{proof}
We start by decomposing the norm of the difference in two components, using triangular inequality:
\begin{align*}
    \norm{\estimatedjacobian(\widehat{\vtheta}) - \jacobian(\vtheta)}_2 
    \le \norm{\estimatedjacobian(\vtheta) - \jacobian(\vtheta)}_2 + \norm{\estimatedjacobian(\vtheta) - \estimatedjacobian(\widehat{\vtheta})}_2.
\end{align*}
The first component is bounded by Lemma~\ref{lemma:boundgradient}. We will bound now the second component, using Reinforce estimator for the gradient:
\begin{align}
    &\norm{\estimatedjacobian(\vtheta) - \estimatedjacobian(\widehat{\vtheta})}_2 =\\ 
    &= \norm{\frac{1}{n}\sum_{i=1}^n\sum_{t=1}^T \nabla_{\vtheta} \log \pi_{\vtheta}(a_{i,t}|s_{i,t}) R_{i,t} \gamma^{t} - \frac{1}{n}\sum_{i=1}^n\sum_{t=1}^T \nabla_{\vtheta} \log \pi_{\hat\vtheta}(a_{i,t}|s_{i,t}) R_{i,t} \gamma^{t}}_2 \\
    &= \frac{1}{n} \norm{\sum_{i=1}^n\sum_{t=1}^T (\nabla_{\vtheta} \log \pi_{\vtheta}(a_{i,t}|s_{i,t}) - \nabla_{\vtheta} \log \pi_{\widehat{\vtheta}}(a_{i,t}|s_{i,t})) R_{i,t}  \gamma^{t}}_2 \\
    &\le \frac{1}{n} \sum_{i=1}^n\sum_{t=1}^T \norm{(\nabla_{\vtheta} \log \pi_{\vtheta}(a_{i,t}|s_{i,t}) - \nabla_{\vtheta} \log \pi_{\widehat{\vtheta}}(a_{i,t}|s_{i,t})) }_2 \norm{  R_t  \gamma^{t}}_2 \\
    &\le \frac{1}{n}\frac{M_R}{(1-\gamma)} \sum_{i=1}^n\sum_{t=1}^T \frac{M_{S}^2}{\sigma^2} \norm{\vtheta -\widehat{\vtheta}}_2  \\
    &\le \frac{T M_{S}^2 M_R}{\sigma^2(1-\gamma)}  \sigma \sqrt{\frac{r + \log(\frac{1}{\delta})}{n \eta}}.
\end{align}
In line (26) we apply the Cauchy-Schwartz inequality. In line (27) we apply lemma \ref{gaussianpolicies} and in line (28) we apply lemma \ref{lemmamle}.
Merging the two results the proof follows.
\end{proof}

\finitesamplegeneral*
\begin{proof}
First we have to bound the error on $\Delta$ created by the behavioral cloning. Given $\Delta = \vtheta_2 - \vtheta_1$ and $\widehat{\Delta} = \widehat{\vtheta}_2 - \hat{\vtheta}_1$:
\begin{equation}\label{differencesdelta}
    \norm{\Delta - \widehat{\Delta}} = \norm{\vtheta_2 - \vtheta_1 - \widehat{\vtheta}_2 + \hat{\vtheta}_1} \\
    \le \norm{\vtheta_1 - \widehat{\vtheta}_1} + \norm{\vtheta_2 - \widehat{\vtheta}_2} \le  2\sigma \sqrt{\frac{r + \log(\frac{1}{\delta})}{n \eta}}.
\end{equation}
So we can bound the difference in norm between the real weights $\vomega^L$ and the estimated weights $\widehat{\vomega}$. We indicate with $\kappa$ the condition number of $\jacobian$, with $\chi = \frac{\norm{\estimatedjacobian - \jacobian}_2}{\norm{\jacobian}_2}$ and $y = \jacobian^{+H}\vomega$.  We apply the pertubation Lemma~\ref{lemma:perturbation}.
\begin{align}
    \norm{\vomega^L - \widehat{\vomega}}_2 &\le \frac{\kappa}{(1- \kappa \chi) \norm{\jacobian}_2} (\chi \norm{\vomega^L}_2 \norm{\jacobian}_2 + \norm{\Delta - \hat{\Delta}}_2) + \chi \norm{y}_2 \norm{\jacobian}_2 \label{eq_1}\\
    &\le \frac{\kappa \norm{\estimatedjacobian - \jacobian}_2}{c \norm{\jacobian}_2} \norm{\vomega^L}_2 + \frac{\kappa \norm{\Delta - \widehat{\Delta}}_2}{c \norm{\jacobian}_2} + \norm{\jacobian}_2 \norm{y}_2 \nonumber \\
    &= \norm{\jacobian - \estimatedjacobian}_2 \left( \frac{\kappa \norm{\vomega^L}_2}{c \norm{\jacobian}_2} + \norm{y}_2\right) + \norm{\Delta - \widehat{\Delta}}_2 \frac{\kappa}{c \norm{\jacobian}_2} \nonumber\\
    &= \norm{\jacobian - \estimatedjacobian}_2 \left( \frac{\norm{\vomega^L}_2}{c \sigma_{\min}(\jacobian)} + \norm{y}_2\right) + \norm{\Delta - \widehat{\Delta}}_2 \frac{1}{c \sigma_{\min}(\jacobian)} \nonumber \\
    &\qquad +  2\sigma \sqrt{\frac{r + \log(\frac{1}{\delta})}{n \eta}} \frac{1}{c \sigma_{\min}(\jacobian)} \\
    &\le O\left(\frac{(M+M^2_{S}M_R)}{\sigma_{\min}(\jacobian)} \sqrt{\frac{\log(\frac{2}{\delta})}{n\eta}}\right),
\end{align}
where in line (29) we apply Lemma~\ref{lemma:perturbation} and in line (30) we apply Equation~\eqref{differencesdelta} and Lemma~\ref{differencesgrad}.
\end{proof}

\newpage
\section{Experiments}
In this appendix, we report some experimental details together with some additional experiments. 
\subsection{Gridworld}
In the Gridworld experiment, we select different learning steps for different learners. The number of learning steps depends on the number of policy updates that the learner takes to become an expert. In the following plots, we report the expected discounted return for each learner: Q-Learning (Figure~\ref{fig:qlearning}), G(PO)MDP (Figure~\ref{fig:gpomdp}), SPI (Figure~\ref{fig:spi}), SVI (Figure~\ref{fig:svi}). In these plots, the expected discounted return is estimating using a batch of $50$ trajectories for each learner. The discount factor used in all experiments is $0.96$.
\label{app:experiments}
\begin{figure}[h!]
\begin{minipage}{0.4\textwidth}
 \centering
     \includegraphics[scale=.8]{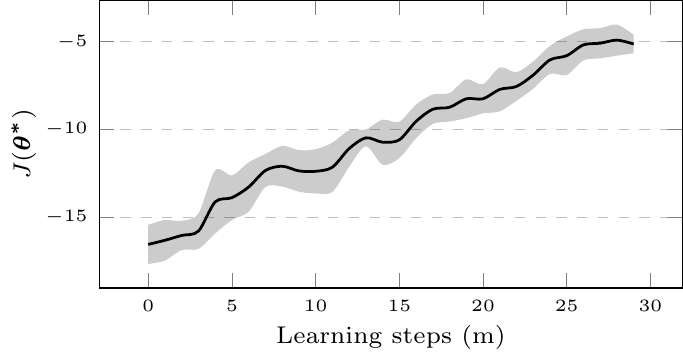}
     \caption{Learning performance of G(PO)MDP. 20 runs, 98\%c.i.}
     \label{fig:gpomdp}
\end{minipage}
\hspace{1cm}
\begin{minipage}{0.4\textwidth}
 \centering
     \includegraphics[scale=.8]{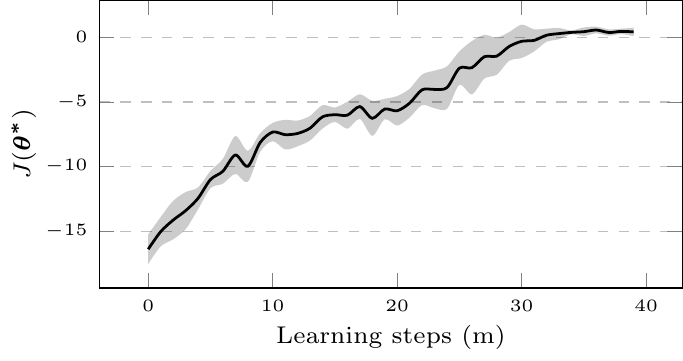}
     \caption{Learning performance of Q-Learning. 20 runs, 98\%c.i.}
     \label{fig:qlearning}
\end{minipage}
\end{figure}
\begin{figure}[h!]
\begin{minipage}{0.4\textwidth}
 \centering
     \includegraphics[scale=.8]{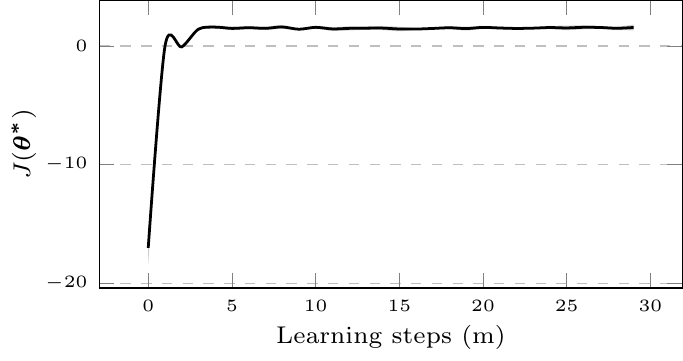}
     \caption{Learning performance of SPI. 20 runs, 98\%c.i.}
     \label{fig:spi}
     \end{minipage}
     \hspace{1cm}
\begin{minipage}{0.4\textwidth}
 \centering
     \includegraphics[scale=.8]{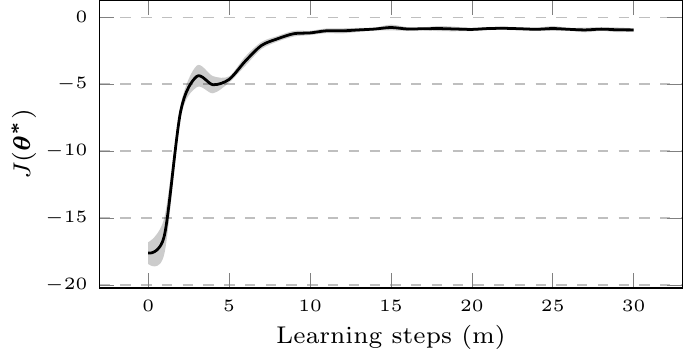}
     \caption{Learning performance of SVI. 20 runs, 98\%c.i.}
     \label{fig:svi}
     \end{minipage}

\end{figure}
\newpage
\subsection{MuJoCo additional experiments}
For the MuJoCo experiments, we use the same hyperparameters as in~\cite{lflpaper}, apart from that we use $16$ parallel agents for PPO, due to resource constraints. The number of forward steps are settled to $2000$. As in~\cite{lflpaper}, we select a subset of the learner's trajectories and we do not use the first $10$ trajectories because the first phase of learning is too noisy. We evaluate the algorithms on the first $1$ million environment steps.

In this section, we report additional results on the comparison between \algname~and another algorithm, T-REX \cite{brown2019extrapolating} which is not created for the LfL setting but which uses suboptimal trajectories. 
The T-REX algorithm aims to recover a reward function from ranked trajectories, where the rank is given by an oracle and is based on the expected discounted return. We use the algorithm in the LfL setting, where we approximate the ranking with the temporal updates of the policies, as was done in an example in the original paper. We use, as in the other MuJoCo experiments, the original reward function of the MuJoCo environment to test the performance of the algorithm and the trajectories from learning step $10$ to learning step $20$. 

We implement the reward function as in the original paper with a three layer neural network with $256$ as hidden size. The T-REX algorithm succeeds in recovering a good approximation of the reward weights in the Hopper domain as it is shown in Figure~\ref{fig:hopper_only_pref}, but not with the same performance of \algname~(see Figure~\ref{fig:hopper_pref}). Instead, T-REX does not succeed into recovering the reward function of the Reacher environment (see Figure~\ref{fig:reacher_pref}). 

It is worth noting that this is an unfair comparison as the T-REX algorithm was not created for the LfL setting.
\begin{figure}[h!]
\begin{minipage}{0.45\linewidth}
    \centering
    \includegraphics[scale=0.9]{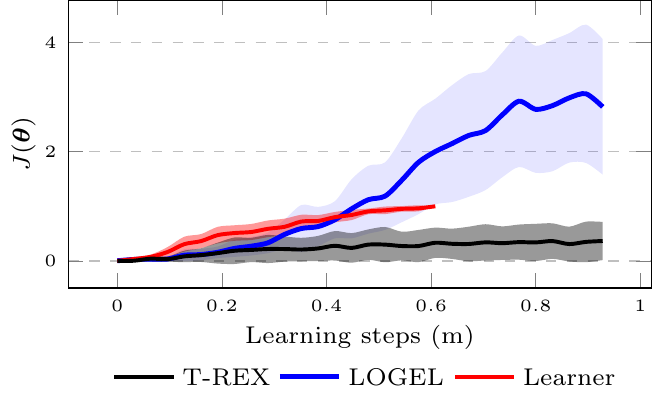}
    \caption{The red line is the performance of the learner on the first $20$ steps of learning. The blue line is the performance of \algname~on the first $30$ steps and the black line the performance of T-REX on the first $30$ steps of learning. The two algorithm are evaluated on the real reward weights. $10$ runs, 98\% c.i. }
    \label{fig:hopper_pref}
\end{minipage}
\hspace{1cm}
\begin{minipage}{0.45\linewidth}
    \centering
    \includegraphics[scale=0.9]{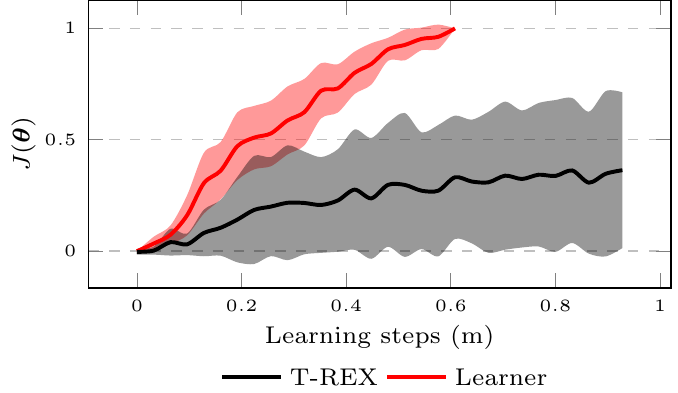}
    \caption{The red line is the performance of the learner on the first $20$ steps of learning. The black line the performance of T-REX on the first $30$ steps of learning. The algorithm is evaluated on the real reward weights.$10$ runs, 98\% c.i.  }    \label{fig:hopper_only_pref}
\end{minipage}
\end{figure}

\begin{figure}[h!]
    \centering
    \includegraphics{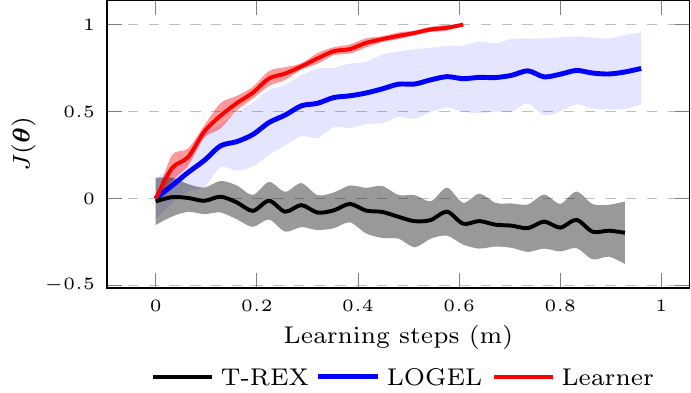}
    \caption{The red line is the performance of the learner on the first $20$ steps of learning. The blue line is the performance of \algname~on the first $30$ steps and the black line the performance of T-REX on the first $30$ steps of learning. The two algorithm are evaluated on the real reward weights.$10$ runs, 98\% c.i.  }
    \label{fig:reacher_pref}
\end{figure}
\newpage
\subsection{Autonomous driving scenario}
In this section, we report an additional, preliminary experiment, that we perform on a simulator driving scenario. We employ SUMO simulator, an open-source, continuous road traffic simulation package designed to handle large road networks. SUMO focuses on the high-level control of the car, integrating an internal system that controls the vehicle dynamics. During the simulation, SUMO provides information on the other vehicles around the ego vehicle. 

We consider a crossroad scenario which consists of an intersection with an arbitrary number of roads. The vehicle coming from the source road has to reach a target road that has a higher priority. The goal of the agent is to drive the ego car and enter the target road, avoiding dangerous maneuvers.

The reward features consists of four components: 
\begin{itemize}
    \item \emph{Time}:  a constant feature at each decision step,
    \item \emph{Jerk}: the absolute value of the instantaneous jerk, i.e., the finite- difference derivative of the acceleration,
    \item \emph{Harsh Slow Down}: a binary feature,  which activates whenever the velocity is lower than a threshold,
    \item \emph{Crash}: a binary feature which activates when the vehicle violates the safety constraints or performs a crash.
\end{itemize}

The agent's policy is a rule-based policy, i.e., a set of parametrized rules, which is learned using Policy Gradients with Parameter-based Exploration (PGPE). It is important to notice that the agent's policy is not differentiable.

We perform $10$ PGPE updates of the agents and then we use the learning trajectories with \algname. In the behavioral cloning step, we use a linear layer to approximate the policy of the learner. 
Table~\ref{tab:rewweights} shows the normalized weights recovered by \algname~and the normalized real weights.  As shown in Table~\ref{tab:rewweights} and in Figure~\ref{fig:sumo}, the reward weights recovered are quite similar to the real ones.

\begin{table}[h!]
    \centering
    \begin{tabular}{c|cccc}
    \toprule
         &  Time & Jerk & Slow & Crash\\
         \hline
        Recovered Weights & 0.0401 & 0.0174 & 0.0000 & 0.9424 \\
        Real Weights & 0.0017 & 0.0003 & 0.0000 & 0.9980 \\
        \toprule
    \end{tabular}
    \caption{Reward weights for the autonomous simulate driving scenario.}
    \label{tab:rewweights}
\end{table}
\begin{figure}[h!]
    \centering
    \includegraphics{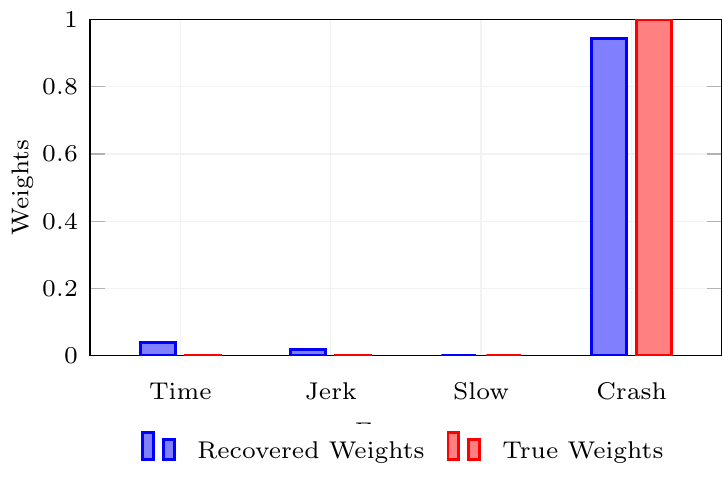}
    \caption{Reward weights for the autonomous simulate driving scenario.}
    \label{fig:sumo}
\end{figure}
\end{document}